\documentclass[11pt]{article}
\usepackage{graphicx} 

\usepackage[left=2.5cm, right=2.5cm, top=2.5cm, bottom=2.5cm]{geometry}
\usepackage{amsfonts,amsthm,amssymb,enumerate,enumitem,verbatim,mathtools,bbm,color,mathrsfs}
\usepackage{hyperref}

\usepackage[numbers]{natbib}

\usepackage{algorithm}
\usepackage{algpseudocode}
\usepackage{listings}
\lstdefinestyle{DOS}
{
    backgroundcolor=\color{black},
    basicstyle=\scriptsize\color{white}\ttfamily
}

\newtheorem{theorem}{Theorem}[section]
\newtheorem{claim}[theorem]{Claim}

\newtheorem{corollary}[theorem]{Corollary}
\newtheorem{definition}[theorem]{Definition}
\newtheorem{lemma}[theorem]{Lemma}
\newtheorem{proposition}[theorem]{Proposition}

\newcommand{\prob}{\mathbb{P}}

\newcommand{\Bin}[2]{\text{Bin}(#1,#2)}
\newcommand{\gstochdom}{\succcurlyeq}
\newcommand{\lstochdom}{\preccurlyeq}

\newcommand{\sfrac}[2]{\textstyle{\frac{#1}{#2}}}

\newcommand{\blank}{\,\cdot\,}

\newcommand{\XG}{\mathcal{X}_G}
\newcommand{\UnivDist}[2]{\text{Univ}(#1,#2)}
\newcommand{\eUnivDist}[2]{\text{\emph{Univ}}(#1,#2)}
\newcommand*{\boldone}{\text{\usefont{U}{bbold}{m}{n}1}}
\newcommand{\ntext}[1]{\text{\normalfont{#1}}}

\newcounter{capitalcounter}
\newcommand{\eref}[1]{\emph{\ref{#1}}}

\renewcommand{\epsilon}{\varepsilon}

\usepackage{pgf,tikz,ifthen}
\usetikzlibrary{arrows,automata,calc,fit}

\definecolor{capri}{rgb}{0.0, 0.50, 1.0}

\tikzset{
    > = stealth, 
    shorten > = 1pt, 
    auto,
    semithick 
}

\tikzstyle{every state}=[
    draw = black,
    thick,
    fill = white,
    minimum size = 8mm
]

\tikzset{
    global edge style/.style={every edge/.append style={thick}},
}

\title{A General Upper Bound for the Runtime of a Coevolutionary Algorithm on Impartial Combinatorial Games}
\date{}
\author{Alistair\ Benford\thanks{School of Computer Science, University of Birmingham,
Birmingham,
B15 2TT,
UK.
a.s.benford@bham.ac.uk}
\ and\ Per\ Kristian\ Lehre\thanks{School of Computer Science, University of Birmingham,
Birmingham,
B15 2TT,
UK.
p.k.lehre@bham.ac.uk
\newline
This research was supported by a Turing AI Fellowship (EPSRC grant ref EP/V025562/1).
}}

\begin{document}

\maketitle

\begin{abstract}
    Due to their complex dynamics, combinatorial games are a key test case and application for algorithms that train game playing agents. Among those algorithms that train using self-play are coevolutionary algorithms (CoEAs). However, the successful application of CoEAs for game playing is difficult due to pathological behaviours such as cycling, an issue especially critical for games with intransitive payoff landscapes. 

    Insight into how to design CoEAs to avoid such behaviours can be provided by runtime analysis. In this paper, we push the scope of runtime analysis for CoEAs to combinatorial games, proving a general upper bound for the number of simulated games needed for UMDA to discover (with high probability) an optimal strategy. This result applies to any impartial combinatorial game, and for many games the implied bound is polynomial or quasipolynomial as a function of the number of game positions. After proving the main result, we provide several applications to simple well-known games: Nim, Chomp, Silver Dollar, and Turning Turtles. As the first runtime analysis for CoEAs on combinatorial games, this result is a critical step towards a comprehensive theoretical framework for coevolution.
\end{abstract}

\section{Introduction}\label{sect:introduction}

Many of the most well-known games in the world are combinatorial games. Combinatorial games are typically perfect-information games played by two players without chance moves. The game has a finite number of possible positions, and players alternately take turns moving the game from one position to another, according to a set of rules describing which moves are legal. Combinatorial games are an exceptionally broad class of games, including famous games enjoyed the world over such as Chess or Go. Even those with simple rules can engender deep and complex strategic interactions between players. While this strategic depth is a key part of the appeal for human players, it can also render the task of computing a winning strategy to be extremely difficult. Indeed, games for which this task is known to be EXPTIME-complete (in terms of board size) include Chess~\cite{FL-exptime-chess}, Go (with the ko rule)~\cite{R-go-exptime}, and Checkers~\cite{R-checkers-exptime}. It is also known that determining an optimal strategy for a poset game (a class of combinatorial games which we will encounter in Section~\ref{sect:chomp}) is PSPACE-complete in terms of the size of the underlying poset~\cite{G-poset-games}. (For further results, see~\cite{DH-algorithmic-combinatorial-game-theory}.)

While classical methods are impractical for such cases, strong strategies can still be developed by using heuristic approaches, such as neural networks, Monte Carlo tree search, or genetic programming. Indeed, combinatorial games are a long-standing focus in the development of artificial intelligence, from Donald Michie's seminal use of reinforcement learning on Tic-Tac-Toe~\cite{M-MENACE}, to Deep Blue's famous matches against then-world Chess champion Garry Kasparov~\cite{D-kasparov-deepblue}, to the recent groundbreaking results of DeepMind~\cite{D-alpha-zero}. Many recent successes in this area train game playing agents using self-play, and among those self-play heuristics are coevolutionary algorithms (CoEAs)~\cite{P-coevolutionary-principles}. For a CoEA, self-play is realised through one or more evolving populations of individuals who compete against their contemporaries. In each iteration, the strongest individuals are selected based on their competitive interactions. Through genetic mutation and crossover, these strongest individuals are then used as parents for the individuals in the next iteration.

The successful application of CoEAs is deeply challenging, often due to the potential for games with intransitive payoff landscapes to induce cyclic behaviour~\cite{F-coevolution-thesis}. For standard evolutionary algorithms, which apply similar methods to traditional optimisation problems, insight into how to avoid pathological behaviours can be provided by runtime analysis, which exists in great breadth and depth in literature and continues to be actively developed~\cite{DN-runtime-analysis-survey}. However, despite clear demand (see~\cite{P-coevolutionary-principles}), runtime analysis that addresses the challenges unique to CoEAs is far more limited. Indeed, while existing coevolutionary runtime analysis concerns a range of algorithms and design features, there are only three problem settings to which it so far applies: \textsc{Bilinear}, a game played on bitstrings whose outcome depends only on the number of $1$-bits selected by each player~\cite{L-pdcoea-theory,HL-fitness-aggregation,HLL-rls-analysis}; \textsc{Diagonal}, a benchmark problem inspired by binary test-based optimisation~\cite{LL-diagonal}; and a class of symmetric zero-sum games with a payoff landscape that is globally very simple, but possibly locally intransitive~\cite{BL-symmetric-zero-sum-games}. Accordingly, our core research aim is to push the scope of runtime analysis for CoEAs towards games which feature more complex strategic interaction between players, and more closely reflect real-world games.

Motivated by the numerous empirical investigations into the topic (see Section~\ref{sect:related-work}), we focus our analysis on the use of CoEAs for combinatorial games, and in particular \emph{impartial} combinatorial games. A combinatorial game is said to be impartial if both players share the same set of available moves at each game position~\cite{G-impartial}. It is common to also adopt the \emph{normal play convention}, which assumes that a player loses if they have no legal moves available. For instance, consider $\textsc{SubtractionNim}_7^2$ (formally introduced in Section~\ref{sect:subtraction-nim}), in which the game positions are $\{0,1,2,3,4,5,6\}$ and a player must subtract either $1$ or $2$ from the position on their turn. A strategy may be encoded as a string of length $6$, with entry $i$ indicating whether a $1$ or a $2$ is to be subtracted when the game position is $i$. One way the game may play out is then:
\begin{minipage}{0.35\textwidth}
\begin{align*}
    \text{Player 1}&:\texttt{122111}\\
    \text{Player 2}&:\texttt{122122}
\end{align*}
\end{minipage}
\begin{minipage}{0.55\textwidth}
\begin{align*}
    &\text{\textbf{Turn 1:} P1 subtracts 1. The new position is 5.}\\
    &\text{\textbf{Turn 2:} P2 subtracts 2. The new position is 3.}\\
    &\text{\textbf{Turn 3:} P1 subtracts 2. The new position is 1.}\\
    &\text{\textbf{Turn 4:} P2 subtracts 1. The new position is 0.}\\
    &\text{\textbf{Turn 5:} P1 has no legal moves. P2 is the winner.}
\end{align*}
\end{minipage}

\vspace{10pt}
\noindent (In fact, any strategy of the form \texttt{12*12*} will always win this game, provided the corresponding player does not move first.)

The main result of this paper (Theorem~\ref{thm:upper-bound-impartial-combinatorial} and Corollary~\ref{cor:impartial-games}) is the first runtime analysis for a coevolutionary algorithm on impartial combinatorial games. In broad terms, it says the following.

\begin{theorem}[Corollary~\ref{cor:impartial-games}, informal version]\label{thm:informal-statement}
    Let $\mathcal{A}$ be the coevolutionary algorithm specified in Section~\ref{sect:umda}, and let $G$ be an impartial combinatorial game with $n$ possible positions. Then, with high probability, $\mathcal{A}$ discovers an optimal strategy for $G$ within $n^{O(\overline{s})}$ game evaluations, where $\overline{s}$ is a precisely defined invariant of the corresponding game graph.
\end{theorem}

We note that the notion of a game graph is defined in Section~\ref{sect:impartial-combinatorial-games} and the invariant $\overline{s}$ is defined in Section~\ref{sect:switchability}. For many games we find $\overline{s}=O(1)$ or $\overline{s}=O(\log{n})$, and so this result implies a range of polynomial and quasipolynomial runtimes. While it appears likely that the upper bound provided is higher than the true runtime for specific games, a major strength is that it is immediately applicable to any impartial combinatorial game. As we also provide an easy method for bounding $\overline{s}$ above when its exact value is not obvious (see Proposition~\ref{prop:switchability-bound}), deriving runtimes for well-known games is straightforward. Indeed, after distilling into a more concise form (Corollary~\ref{cor:impartial-games}), we will see applications to games including Nim, Silver Dollar, Turning Turtles, and Chomp.

To understand what is the significance of our result, it is helpful to first clarify what it is not. In no uncertain terms, this paper is not an account of a superior ready-to-use method for efficiently finding optimal strategies for combinatorial games. Strategies will here be encoded by exhaustively listing a preferred action for every possible game position, and thus the methods presented are necessarily at least linear in the number of game states, both in terms of memory and of time. With this naive representation, classical algorithms can already establish optimal strategies in time $O(n)$ using Sprague-Grundy theory (see Section~\ref{sect:sprague-grundy}), which is best possible. However, many games are parameterised in such a way that the number of possible game positions grows exponentially (accordingly, we emphasise that Theorem~\ref{thm:informal-statement} is not in contradiction with the aforementioned EXPTIME and PSPACE results). While the classical approach breaks down in such cases, a CoEA can still find success by replacing the exhaustive listing of actions with a model that maps features of the game position onto an action.

However, even when using the naive representation, our understanding of how to successfully apply CoEAs is very limited. If we wish to consistently apply CoEAs to advanced problems, whether they are rooted in game-playing or not, we must attain a comprehensive understanding of their behaviour in these simpler settings. Indeed, seemingly simple instances still produce payoff landscapes with features that make them difficult to optimise heuristically, such as intransitivity (as an example, in the already-introduced representation for $\textsc{SubtractionNim}_7^2$, $\texttt{111121}$ defeats $\texttt{122112}$, which in turn defeats $\texttt{12122}$, which in turn defeats $\texttt{111121}$, regardless of who plays first).

Thus, the main contribution of this paper is precisely this: a first step towards a theoretical understanding of CoEAs on combinatorial games. This greatly expands the scope of rigorous runtime analysis available for coevolution (which so far does not apply to any turn-based game, let alone combinatorial ones), and additionally complements the abundance of existing empirical analysis, which we review in Section~\ref{sect:related-work}. While it remains a long term goal to push analysis towards more sophisticated representations, insights into algorithm design gained here still hold great relevance to coevolution in general. Furthermore, we believe our addition to the range of techniques available in this critical domain will in turn further the development of future runtime analysis of CoEAs.

Finally, we note here that the algorithm we analyse is a type of coevolutionary algorithm called an estimation of distribution algorithm (EDA), and moreover that this EDA applies to multi-valued decision variables (these notions are covered in Section~\ref{sect:umda}). While not the main focus of this paper, there is only a small amount of preexisting analysis for EDAs operating over non-binary search domains, despite the clear utility of such algorithms. Our proof includes a detailed treatment of this setting, and may also provide methods useful in future analysis in this area.

In the remainder of this section, we review existing related work before stating notation. In Section~\ref{sect:impartial-combinatorial-games} we give a more comprehensive discussion of impartial combinatorial games and review some Sprague-Grundy theory that will be relevant to our proof. In Section~\ref{sect:umda} we state the algorithm to which our result applies (UMDA), with an emphasis on its extension to multi-valued decision variables. In Section~\ref{sect:switchability} we motivate and define the graph property $\overline{s}$ appearing in Theorem~\ref{thm:informal-statement}, before then presenting the main result in Section~\ref{sect:main-result}. Following this, we apply the main result to a menagerie of selected impartial combinatorial games in Section~\ref{sect:applications}.

\subsection{Related work}\label{sect:related-work}

\textbf{Empirical analysis of coevolutionary algorithms for game playing.} As game playing is a natural application for CoEAs, there have been a large number of empirical investigations into this topic, of which we can only list a small fraction here. In terms of impartial combinatorial games, Rosin and Belew~\cite{RB-coevolution-nim-3d-tic-tac-toe} investigated the effect of using features such as fitness sharing and archives in CoEAs optimising a 4-pile instance of Nim, noting that Nim was a difficult coevolutionary problem despite lending itself to simple crossover-friendly representations. Additionally, J\'{a}kowski, Krawiec, and Wieloch~\cite{JKW-coevolution-nim-tic-tac-toe} observed in relation to experiments on $\textsc{SubtractionNim}_{200}^3$ that intransitivity presents a strong challenge for CoEAs. Non-impartial (yet still almost symmetric) combinatorial games studied in the context of coevolution include Tic-Tac-Toe~\cite{JKW-coevolution-nim-tic-tac-toe,RB-coevolution-nim-3d-tic-tac-toe}, Backgammon~\cite{PB-backgammon}, Othello~\cite{JSL-othello,SJLK-othello,SJK-othello}, Senet~\cite{FM-senet}, Checkers~\cite{CF-checkers}, Chess~\cite{F-et-al-chess,HS-chess}, and Go~\cite{LM-go}. More general game-playing applications include Pong~\cite{MSM-pong}, Bomberman~\cite{G-et-al-bomberman}, Poker~\cite{NW-poker}, Resistance~\cite{LSE-resistance}, as well as games invented to emulate real-world applications such as cyber security and defense~\cite{HT-defense-and-security,L-et-al-defendit}. For a general survey, see~\cite{KH-complex-problems-coevolution}.

\textbf{Runtime analysis of coevolutionary algorithms.} Until recently, the only existing coevolutionary runtime analysis result, due to Jansen and Wiegand~\cite{JW-cooperative-theory}, applied to a \emph{cooperative} coevolutionary algorithm, which uses multiple populations to collectively solve traditional optimisation problems. The first runtime analysis applicable to competitive coevolution was established by Lehre~\cite{L-pdcoea-theory}, who showed that a population-based CoEA which selects using a pairwise dominance relation is able to approximate the Nash equilibrium of instances of a game called $\textsc{Bilinear}$ in expected polynomial time. A key theoretical insight into algorithm design from the same paper was the identification of an error threshold for mutation rate, above which no CoEA can efficiently optimise $\textsc{Bilinear}$. Further runtime analysis for CoEAs on $\textsc{Bilinear}$ has concerned the roles played by fitness aggregation methods~\cite{HL-fitness-aggregation} and archives~\cite{HLL-rls-analysis} in algorithm behaviour. Inspired by promising applications of CoEAs for optimising binary test-based problems, Lin and Lehre~\cite{LL-diagonal} provided runtime analysis establishing the benefit of using a CoEA over a traditional EA for optimising a benchmark problem called $\textsc{Diagonal}$. In~\cite{BL-symmetric-zero-sum-games}, Benford and Lehre considered the importance of maintaining a diverse set of opponents when coevolving game strategies, showing that any CoEA able to retain only one individual between generations cannot efficiently find optimal strategies on a certain class of symmetric zero-sum games, even though with high probability a coevolutionary EDA finds an optimal strategy in polynomial time.

\subsection{Notation}\label{sect:notation}

Given a finite set $S$, a \emph{probability distribution over $S$} is a function $p:S\to[0,1]$ satisfying $\sum_{s\in S}p(s)=1$. We say that an $S$-valued random variable $x$ is \emph{distributed according to $p$}, written $x\sim p$, if $\prob(x=s)=p(s)$ holds for every $s\in S$. Given also a subset $A\subseteq S$, we write $p(A)=\sum_{s\in A}p(s)$. Given a number $\gamma\in[0,1]$ we use $\mathcal{P}_\gamma(S)$ to denote the set of probability distributions $p$ over $S$ satisfying $p(s)\geqslant\gamma$ for every $s\in S$, and we also write $\mathcal{P}(S)=\mathcal{P}_0(S)$.

A rooted directed graph is a triple $G=(V,F,v_0)$, were $V$ is a vertex set, $F$ is a function mapping each vertex onto its out-neighbourhood, and $v_0\in V$ is a distinguished root vertex. Throughout we will assume all directed graphs are acyclic. We write $E(G)=\{(u,v)\in V^2:v\in F(u)\}$ for the set of edges of $G$ and $\Delta=\max_{v\in V}|F(v)|$ for the maximum degree of $G$. A directed path in $G$ is a sequence of vertices $u_0u_1\ldots u_\ell$ such that $u_i\in F(u_{i-1})$ for each $i\in[\ell]$. For a path $P=u_0u_1\ldots u_\ell$ we have $|P|=\ell+1$. If $v\in V$ has no out-neighbours, then we say $v$ is a \emph{sink}. We use $\text{Int}(G)=\{v\in V:F(v)\neq\emptyset\}$ to denote the set of non-sink vertices of $G$ (the \emph{interior} vertices).

All logarithms are the natural logarithm unless stated otherwise, and given $k\in\mathbb{N}$ we write $\log^k{n}=(\log{n})^k$.

\section{Impartial combinatorial games}\label{sect:impartial-combinatorial-games}

Let us briefly review the representation of impartial games via directed graphs and some Sprague-Grundy theory (see, for example,~\cite{G-what-is,N-sprague-grundy}). An \emph{impartial combinatorial game} is a finite acyclic rooted directed graph $G=(V,F,v_0)$ (see Section~\ref{sect:notation}), where $V$ is a vertex set of size $n$, and $v_0\in V$ is the initial game position. Players take it in turns to move the current position to one of its out-neighbours. We adopt the convention that if a player is unable to make a move because the current position has no out-neighbours (i.e., it is a sink), then that player loses. This is usually referred as the \emph{normal play convention}. We will also always assume that for each $v\in V$, there is a directed path from $v_0$ to $v$, so that every game position is reachable.

We will encode strategies for impartial combinatorial games as an assignment of each non-sink game position $v$ to an element of $F(v)$ (that is, an out-neighbour of $v$), with this assignment indicating the preferred move at each game position. Formally, recalling that $\text{Int}(G)$ denotes the set of $v\in V$ with $F(v)\neq\emptyset$, then
\begin{equation*}
    \XG=\prod_{v\in\text{Int}(G)}F(v)
\end{equation*}
will be the set of strategies for $G$. Note that an element $x\in\XG$ may be regarded as a mapping $\ntext{Int}(G)\to V$, and so we will write $x(v)$ for the image of a position $v\in V$ under this mapping. This formulation coincides closely with that featured in the aforementioned work of Richie on reinforcement learning for optimal Tic-Tac-Toe play~\cite{M-MENACE}, and has similarities to subsequent `move selector' representations which identify a preferred action based on the current game position using, for example, genetic programming~\cite{G-et-al-bomberman}, neural networks~\cite{LM-go,MSM-pong}, or a game-specific mapping~\cite{NW-poker,LSE-resistance}. However, it stands distinct from `state evaluator' representations which play by evaluating board positions, whether by recording evaluations for all possible positions~\cite{JKW-coevolution-nim-tic-tac-toe,RB-coevolution-nim-3d-tic-tac-toe}, genetic programming~\cite{H-ec-games-thesis,FM-senet,HS-chess}, neural networks~\cite{PB-backgammon,SJK-othello,CF-checkers,F-et-al-chess}, or otherwise.

As is typical for the uses of coevolution for gameplaying discussed in Section~\ref{sect:related-work}, players receive a payoff depending only on whether the final outcome of the game was win or lose. Accordingly, let $f_G:\XG\times\XG\to\{-1,1\}$ be the payoff function for $G$, where $f_G(x,y)=1$ indicates that $x$ wins against $y$ and $f_G(x,y)=-1$ indicates that $x$ loses against $y$ (where $x$ makes the first move). Precisely, if we recursively define for $v\in V$,
\begin{equation*}
    f_G^v(x,y)=
    \begin{cases}
        -f_G^{x(v)}(y,x) & \qquad\text{if $v\in\text{Int}(G)$,}\\
        -1 & \qquad\text{otherwise,}
    \end{cases}
\end{equation*}
then $f_G(x,y)=f_G^{v_0}(x,y)$. It will also be convenient to define for $x,y\in\XG$,
\begin{equation*}
    \text{Path}_G(x,y)=\{v_0,x(v_0),y(x(v_0)),x(y(x(v_0))),\ldots\}.
\end{equation*}
We will always assume that there is some $x\in\XG$ such that $f_G(x,y)=1$ for every $y\in\XG$ (i.e., that the first player has a winning strategy for $G$). Indeed, if this is not the case, then the second player has a winning strategy, and so we can add a fictitious initial position $v^\ast$ to $G$ with $F(v^\ast)=\{v_0\}$ to obtain a game equally challenging as $G$ but with a winning strategy for the first player. We thus define the set of \emph{optimal strategies for $G$} to be
\begin{equation*}
    \text{Opt}(G)=\{x\in\XG:\text{$f_G(x,y)=1$ for every $y\in\mathcal{X}$}\},
\end{equation*}
and remark that the above assumption implies that $\text{Opt}(G)$ will always be non-empty.

\subsection{The Sprague-Grundy function}\label{sect:sprague-grundy}

First introduced independently by Sprague~\cite{S-mathematische,S-nim} and Grundy~\cite{G-math-and-games}, the Sprague-Grundy function of an impartial combinatorial game is a function mapping game positions onto non-negative integers, which contains information about the game's strategic landscape and how optimal play is affected when building new games out of smaller ones~\cite{F-games-of-no-chance-chapter}. Formally, given $G=(V,F,v_0)$, the Sprague-Grundy function $h:V\to\mathbb{N}_0$ is defined recursively. First, all sink vertices are given the value $0$. Then, once all out-neighbours of $v$ have a value assigned, we define
\newcommand{\mex}[1]{\ntext{mex}\,{#1}}
\begin{equation*}
    h(v)=\mex{\{h(w):w\in F(v)\}},
\end{equation*}
where $\mex{S}=\min{(\mathbb{N}_0\setminus S)}$ denotes the smallest non-negative number not in a finite set $S$ (the `minimum excluded integer').

Given $v\in V$, if the current position is $v$ then the player making the next move has a winning strategy if and only if $h(v)\neq 0$. Accordingly, if $h(v)=0$ then the player making the next move will always lose against an opponent who plays optimally. Thus, victory can be assured for the player making the first move by always choosing to move to vertices in $h^{-1}(\{0\})$. Because this happens automatically whenever $F(v)\setminus h^{-1}(\{0\})=\emptyset$, an optimal strategy can be guaranteed by learning optimal moves at a set $W_G$ (which we refer to as \emph{critical positions}) defined in the following way.

\begin{definition}\label{def:critical-positions}
    Given an impartial combinatorial game $G=(V,F,v_0)$, let 
    \begin{equation*}
        W_G=\{v\in\text{\emph{Int}}(G):\text{$h(v)\neq0$ and $F(v)\setminus h^{-1}(\{0\})\neq\emptyset$}\},
    \end{equation*}
    where $h:V\to\mathbb{N}_0$ is the Sprague-Grundy function for $G$. 
\end{definition}

The following lemma formalises this notion in a general form that will be useful to quote later.

\begin{lemma}\label{lm:optimality-characterisation}
    Let $h:V\to\mathbb{N}_0$ denote the Sprague-Grundy function of a combinatorial game $G$. Let $u_1,\ldots,u_n$ be an ordering of $V$ such that $F(u_i)\subseteq \{u_{1},\ldots,u_{i-1}\}$ for every $i\in[n]$. Then, the following holds for every $i\in[n]$.
    \stepcounter{capitalcounter}
    \begin{enumerate}[label = {\emph{\bfseries \Alph{capitalcounter}\arabic{enumi}}}]
        \item\label{oc1} If $h(u_i)\neq 0$ and $x\in\XG$ satisfies $h(x(v))=0$ for every $v\in W_G\cap\{u_1,\ldots,u_i\}$, then $f_G^{u_i}(x,y)=1$ holds for every $y\in\XG$.
        \item\label{oc2} If $h(u_i)=0$ and $y\in\XG$ satisfies $h(y(v))=0$ for every $v\in W_G\cap\{u_1,\ldots,u_i\}$, then $f_G^{u_i}(x,y)=-1$ holds for every $x\in\XG$.
    \end{enumerate}
    In particular, with our assumption that the first player always has a winning strategy for $G$, if $x\in\XG$ satisfies $h(x(v))=0$ for every $v\in W_G$, then $x\in\text{\emph{Opt}}(G)$.
\end{lemma}
\begin{proof}
    We prove that the conditions \eref{oc1} and\eref{oc2} always hold by induction on $i$. For the case $i=1$, note that we must have $h(u_1)=0$ (as $u_i\notin\text{Int}(G)$) and $f_G^{u_1}(x,y)=-1$ for any $x,y\in\XG$. For the inductive stage, there are two cases to consider. First, if $h(u_i)=0$ and $y\in\XG$ satisfies $h(y(v))=0$ for every $v\in W_G\cap\{u_1,\ldots,u_i\}$, then because $\mex{\{h(w):w\in F(u_i)\}}=0$ we must have $h(x(u_i))\neq0$ for any $x\in\XG$, and hence
    \begin{equation*}
        f_G^{u_i}(x,y)=-f_G^{x(u_i)}(y,x)\overset{\eref{oc1}}{=}-1.
    \end{equation*}
    On the other hand, if $h(u_i)\neq 0$ and $x\in\XG$ satisfies $h(x(v))=0$ for every $v\in W_G\cap\{u_1,\ldots,u_i\}$, then in fact $h(x(u_i))=0$ (for this holds by default if $u_i\notin W_G$), and so for any $y\in\XG$,
    \begin{equation*}
        f_G^{u_i}(x,y)=-f_G^{x(u_i)}(y,x)\overset{\eref{oc2}}{=}1,
    \end{equation*}
    as required.
\end{proof}

Note that the final conclusion of Lemma~\ref{lm:optimality-characterisation} is a sufficient condition, but not a necessary condition, as demonstrated by Figure~\ref{fig:optimality-example}.

\begin{figure}
    \begin{center}
    \begin{tikzpicture}[
        node distance = 3cm, 
        global edge style
    ]

        \node[state] (v0) {$v_0$} ;
        \node[state] (a) [above of=v0] {$a$} ;
        \node[state] (b) [right of=a] {$b$} ;
        \node[state] (c) [right of=b] {$c$} ;
        \node[state] (d) [below of=c] {$d$} ;

        \node[red, above left=7pt] at (v0) {1};
        \node[red, above left=7pt] at (a) {0};
        \node[red, above left=7pt] at (b) {2};
        \node[red, above right=7pt] at (c) {1};
        \node[red, above right=7pt] at (d) {0};

        \path[->] (v0) edge (a) ;
        \path[->] (v0) edge (b) ;
        \path[->] (v0) edge (d) ;
        \path[->] (a) edge (b) ;
        \path[->] (b) edge (c) ;
        \path[->] (b) edge (d) ;
        \path[->] (c) edge (d) ;
        \path[->] (v0) edge (a) ;
        
    \end{tikzpicture}
    \end{center}
    \caption{In the combinatorial game illustrated above, Sprague-Grundy values at each game position are shown in red. In this game, $W_G=\{v_0,b\}$. However, any strategy $x$ with $x(v_0)=d$ is automatically optimal (the first player wins on their first turn), and so the condition of Lemma~\ref{lm:optimality-characterisation} is not a necessary one.}
    \label{fig:optimality-example}
\end{figure}
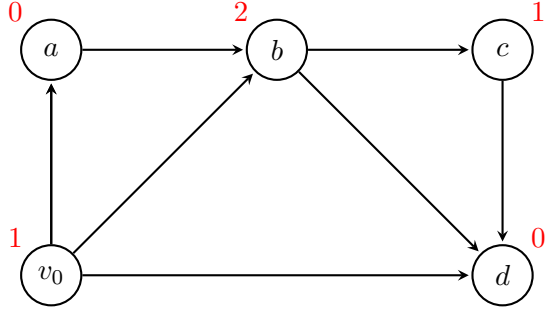

\section{UMDA}\label{sect:umda}

Rather than storing a population as a set of points in the search space, as is the case for most EAs, an \emph{estimation of distribution algorithm} (EDA) represents its population as a probability distribution over the search space~\cite{PHL-eda}. Whereas most algorithms sample candidates for selection from their current population uniformly at random, an EDA instead samples from its probability distribution. After selection has been completed, the selected individuals are then used to update the probability distribution for the next generation. Much of the existing runtime analysis for EDAs (see~\cite{DL-EDAs,D-doerr-cGA,D-cGA,W-umda-runtime,W-cGA}) has emphasised the benefit provided by a high level of diversity among generated search points. This is also the case in the recent first runtime analysis of a coevolutionary EDA~\cite{BL-symmetric-zero-sum-games}, wherein the difficulty presented by locally intransitive payoff landscapes could be provably averted by evaluating strategies against a diverse set of opponents. As intransitivity is also apparent in impartial combinatorial games, a coevolutionary EDA is a good candidate for a first runtime analysis on this topic too.

Most existing theoretical analysis of EDAs concerns those operating over bitstrings -- that is, $\{0,1\}^n$ is the search domain. However, as outlined in Section~\ref{sect:impartial-combinatorial-games}, our formulation of strategies gives rise to a more complicated search domain. For a parent set $S$, we are considering search domains of the form $\mathcal{X}=\prod_{i\in I}S_i$, where $I$ is an indexing set and $S_i\subseteq S$ for each $i\in I$. Given a tuple $p\in\prod_{i\in I}\mathcal{P}(S_i)$, let $\UnivDist{\mathcal{X}}{p}$ denote the probability distribution over $\mathcal{X}$ such that if $x\sim\UnivDist{\mathcal{X}}{p}$ then for any $y\in\mathcal{X}$,
\begin{equation*}
    \prob(x=y)=\prod_{i\in I}p(i)(y_i),
\end{equation*}
so that the distribution of $x$ is that of an independent univariate sampling for each $i\in I$. For notational convenience, given a tuple $p\in\prod_{i\in I}\mathcal{P}(S_i)$ we will often write for $i\in I$ and $s\in S$,
\begin{equation*}
    p(i,s)=
    \begin{cases}
        p(i)(s)&\qquad\text{if $s\in S_i$,}\\
        0&\qquad\text{otherwise.}
    \end{cases}
\end{equation*}
The coevolutionary EDA we consider will represent its current population as an element $p\in \prod_{i\in I}\mathcal{P}(S_i)$, with individuals being generated according to $\UnivDist{\mathcal{X}}{p}$. In the case where $I=[n]$ and $S_i=\{0,1\}$ for each $i\in[n]$, we recover the standard framework for univariate EDAs operating over bitstrings. For these EDAs, the tuple $p\in\prod_{i\in[n]}\mathcal{P}(\{0,1\})$ is often represented as a frequency vector $(p(1),\ldots,p(n))\in[0,1]^n$, where $p(i)$ is the probability that $x\sim\UnivDist{\{0,1\}^n}{p}$ has a $1$-bit in position $i$. A common feature for EDAs operating over bitstrings is to constrain these frequencies to the interval $[\gamma,1-\gamma]$ for some small $\gamma$ at the end of each generation. For the general case, where we track a tuple $p\in\prod_{i\in I}\mathcal{P}(S_i)$, we need to constrain each $p(i)\in\mathcal{P}(S_i)$ to the set $\mathcal{P}_\gamma(S_i)$. To achieve this, we adopt the following minor variation of the multi-valued EDA framework proposed by Ben Jedidia, Doerr, and Krejca~\cite{BDK-multivalued-EDA}. Given $\gamma\in[0,\sfrac{1}{|S|})$ and $p\in\mathcal{P}(S)$, let
\begin{equation*}
    \beta^+_\gamma(p)=\sum_{s\in S}\max{\{p(s)-\gamma,0\}},\qquad\beta^-_\gamma(p)=\sum_{s\in S}\max{\{\gamma-p(s),0\}},
\end{equation*}
Let $\pi_\gamma^S:\mathcal{P}(S)\to\mathcal{P}_\gamma(S)$ then be the function given by
\begin{equation*}
    \pi_\gamma^S(p)(s)=
    \begin{cases}
        \gamma & \qquad\text{if $p(s)\leqslant\gamma$,}\\
        \gamma+\left(1-\sfrac{\beta_\gamma^-(p)}{\beta_\gamma^+(p)}\right)(p(s)-\gamma) & \qquad \text{if $p(s)\geqslant\gamma$.}
    \end{cases}
\end{equation*}
For the case $|S|=2$ the definition reduces to $\pi_\gamma^S(p)(s)=\min{\{\max{\{p(s),\gamma\}},1-\gamma\}}$, and so this model fits the usual method for constraining univariate EDAs over bitstrings.

Despite some differences in notation, the function $\pi_\gamma^S$ is nearly identical the restriction described in~\cite{BDK-multivalued-EDA}. In the context of \cite[Section~4.2]{BDK-multivalued-EDA}, our only modification is to forego an initial clamping of probabilities to the interval $[\gamma,1-(|S|-1)\gamma]$, as the upper border of $1-(|S|-1)\gamma$ is already implied by the fact that the remaining steps produce an element of $\mathcal{P}_\gamma(S)$. Indeed, an actual difference between the two methods only arises for inputs $p$ satisfying $\max_{s\in S}p(s)>1-(|S|-1)\gamma$, and even in such cases the difference is not significant.

The fact that $\pi_\gamma^S$ always outputs an element of $\mathcal{P}_\gamma(S)$ is verified by \eref{pi-0} in the following lemma, which also establishes several further properties of $\pi_\gamma$ which will be useful for our later proofs.

\begin{lemma}\label{lm:pi-constrain} 
    Let $\beta_\gamma^+$, $\beta_\gamma^-$, and $\pi_\gamma^S$ be as defined in Section~\ref{sect:umda}. Then, the following properties hold.
    \stepcounter{capitalcounter}
    \begin{enumerate}[label = {\emph{\bfseries \Alph{capitalcounter}\arabic{enumi}}}]
        \item\label{pi-0} For any $p\in\mathcal{P}(S)$, $\sum_{s\in S}\pi_\gamma^S(p)(s)=1$.
        \item\label{pi-1} If $p(s)\geqslant\gamma$, then $\left(1-\sfrac{\beta_\gamma^-(p)}{1-\gamma|S|}\right)p(s)\leqslant\pi_\gamma^S(p)(s)\leqslant p(s)$.
        \item\label{pi-2} For any $S_i\subseteq S$, $p\in\mathcal{P}(S_i)$ and $s\in S$, $\pi_\gamma^{S_i}(p)(s)\leqslant\max{\{\gamma,p(s)\}}$.
        \item\label{pi-3} For any $S_i,A\subseteq S$ and $p\in\mathcal{P}(S_i)$, $\pi_\gamma^{S_i}(p)(A)\leqslant p(A)+\gamma|S_i|$.
    \end{enumerate}
\end{lemma}
\begin{proof}
    We first note that the definitions of $\beta_\gamma^+$ and $\beta_\gamma^-$ imply that for any $\gamma\in[0,1/|S|)$ and $p\in\mathcal{P}(S)$,
    \begin{align}
        \beta_\gamma^+(p)-\beta_\gamma^-(p)&=\sum_{s\in S}(\max{\{p(s)-\gamma,0\}}-\max{\{\gamma-p(s),0\}}) \nonumber\\
        &=\sum_{s\in S}(\max{\{p(s)-\gamma,0\}}+\min{\{p(s)-\gamma,0\}})=\sum_{s\in S}(p(s)-\gamma)=1-\gamma |S|. \label{eq:beta3}
    \end{align}
    Because $1-\gamma|S|>0$ it immediately follows from \eqref{eq:beta3} that
    \begin{equation}\label{eq:beta4}
        \beta_\gamma^-(p)<\beta_\gamma^+(p).
    \end{equation}
    With these observations, we are now ready to prove the desired properties.
    
    \textbf{\eref{pi-0}:} If $p\in\mathcal{P}(S)$, then setting $S^+=\{s\in S:p(s)\geqslant\gamma\}$ and $S^-=S\setminus S^+$ we have
    \begin{align*}
        \sum_{s\in S}\pi_\gamma^S(p)(s)=\gamma |S|+\sum_{s\in S^+}\left(1-\sfrac{\beta_\gamma^-(p)}{\beta_\gamma^+(p)}\right)(p(s)-\gamma)=\gamma |S|+\left(\sfrac{\beta_\gamma^+(p)-\beta_\gamma^-(p)}{\beta_\gamma^+(p)}\right)\beta_\gamma^+(p)\overset{\eqref{eq:beta3}}{=}1.
    \end{align*}
    
    \textbf{\eref{pi-1}:} If $p(s)\geqslant\gamma$, then by setting $\alpha=\beta_\gamma^-(p)/\beta_\gamma^+(p)$,
    \begin{align*}
        \left(1-\sfrac{\beta_\gamma^-(p)}{1-\gamma|S|}\right)p(s)&\overset{\eqref{eq:beta3}}{\leqslant}\left(1-\sfrac{\beta_\gamma^-(p)}{\beta_\gamma^+(p)}\right)p(s)=(1-\alpha)p(s)\leqslant(1-\alpha)p(s)+\alpha\gamma=\gamma+(1-\alpha)(p(s)-\gamma)\\
        &=\pi_\gamma^S(p)(s)=\gamma+\left(1-\sfrac{\beta_\gamma^-(p)}{\beta_\gamma^+(p)}\right)(p(s)-\gamma)\overset{\eqref{eq:beta4}}{\leqslant}\gamma+(p(s)-\gamma)=p(s),
    \end{align*}
    and so \eref{pi-1} holds.
    
    \textbf{\eref{pi-2}:} If $p(s)\leqslant\gamma$ then $\pi_\gamma^{S_i}(p)(s)\leqslant\gamma=\max{\{\gamma,p(s)\}}$. On the hand, if $p(s)\geqslant\gamma$, then \eref{pi-1} implies that $\pi_\gamma^{S_i}(p)\leqslant p(s)=\max{\{\gamma,p(s)\}}$. In either case, \eref{pi-2} holds.

    \textbf{\eref{pi-3}:} We can compute
    \begin{align*}
        \pi_\gamma^{S_i}(p)(A)&=\pi_\gamma^{S_i}(p)(A\cap S_i)=\sum_{s\in A\cap S_i}\pi_\gamma^{S_i}(p)(s)\overset{\eref{pi-2}}{\leqslant}\sum_{s\in A\cap S_i}\max{\{\gamma,p(s)\}}\leqslant\sum_{s\in A\cap S_i}(p(s)+\gamma)\\
        &=p(A)+\gamma|A\cap S_i|\leqslant p(A)+\gamma|S_i|,
    \end{align*}
    as required.
\end{proof}

{\centering
\begin{minipage}{.8\linewidth}
\begin{algorithm}[H]
\caption{UMDA with binary tournament selection}\label{alg:UMDA}
\begin{algorithmic}[1]

\Require{Search domain $\mathcal{X}=\prod_{i\in I}S_i$}.
\Require Function $f:\mathcal{X}\times\mathcal{X}\to\{-1,1\}$.
\Require Algorithm parameters $\mu\in\mathbb{N}$ and $\gamma>0$.

\For{$i\in I$}
    \For{$s\in S_i$}
        \State Set $p_{0}(i)(s)=\frac{1}{|S_i|}$.
    \EndFor
\EndFor

\For{$t\in\mathbb{N}$ until termination criterion met}
\For{$j\in[\mu]$} \label{ln:selection-loop-start}
\State Sample $x\sim\UnivDist{\mathcal{X}}{p_t}$
\State Sample $y\sim\UnivDist{\mathcal{X}}{p_t}$
\If{$f(x,y)=1$}
    \State Set $P_{t+1}(j)=x$
\ElsIf{$f(x,y)=-1$}
    \State Set $P_{t+1}(j)=y$
\EndIf \label{ln:selection-loop-end}
\EndFor

\For{$i\in I$}
    \For{$s\in S_i$}
    \State Set $q_{t+1}(i)(s)=\frac{1}{\mu}|\{j:\text{$P_{t+1}(j)$ has an $s$ in position $i$}\}|$
    \EndFor
\State\label{ln:constraint} Set $p_{t+1}(i)=\pi_\gamma^{S_i}(q_{t+1}(i))$
\EndFor

\EndFor
\end{algorithmic}
\end{algorithm}
\end{minipage}
\par
}

A description of the algorithm we analyse is now provided by Algorithm~\ref{alg:UMDA}, which effectively generalises the version appearing in~\cite{BL-symmetric-zero-sum-games} (which applied only to bitstrings and omitted the step involving $\pi_\gamma^{S_i}$). Note that due to the use of $\pi_\gamma^{S_i}$ in line~\ref{ln:constraint}, we always have $p_t\in\prod_{i\in I}\mathcal{P}_\gamma(S_i)$.

A key step towards analysing the performance of Algorithm~\ref{alg:UMDA} on impartial combinatorial games is understanding the distribution of a selected individual $P_{t+1}(j)$. This will be handled by the following lemma. Its conclusion gives an exact expression for how the probability a selected individual would choose to move from $u$ to $v$ compares to the probability a sampled individual would choose to move from $u$ to $v$ (where a selected individual is simply the winner of a game played between two independent sampled individuals).

\begin{lemma}\label{lm:sampled-individual}
    Let $G$ be an impartial combinatorial game, and let $p\in\prod_{v\in\text{\emph{Int}}(G)}\mathcal{P}(F(v))$. Suppose that $x,y\sim\eUnivDist{\mathcal{X}_G}{p}$ are independent, and
    \begin{equation*}
        z=
        \begin{cases}
            x&\qquad\text{if $f_G(x,y)=1$,}\\
            y&\qquad\text{if $f_G(x,y)=-1$.}
    \end{cases}
    \end{equation*}
    Then, for any $u\in V$ and $v\in F(u)$,
    \begin{equation}\label{eq:sampled-individual}
        \prob(z(u)=v)=p(u,v)\cdot[1+\prob(u\in\text{\emph{Path}}_G(x,y))\cdot(1-\prob(f_G^v(x,y)=1)-\prob(f_G^u(x,y)=1))].
    \end{equation}
\end{lemma}
For an intuition behind \eqref{eq:sampled-individual}, the comparative factor has effectively three terms (here interpreted in the context of Algorithm~\ref{alg:UMDA}):
\begin{itemize}
    \item $\prob(u\in\text{{Path}}_G(x,y))$, the probability the algorithm encounters position $u$;
    \item $\prob(f_G^u(x,y)=1)$, the probability that $u$ is observed as a winning position; and
    \item $1-\prob(f_G^v(x,y)=1)$, the probability that $v$ is observed as a losing position.
\end{itemize}
If it is likely for $v$ to be observed as a losing position, but unlikely for $u$ to be observed as a winning position, then it is beneficial to deliberately move from $u$ to $v$ (placing your opponent in a likely losing position) rather than play out with whatever the current strategy is from $u$ (where you are unlikely to win), thus incurring an increase in the prevalence of $z(u)=v$ among selected individuals. On the other hand, if the reverse is true, then it is beneficial to deliberately avoid moving from $u$ to $v$ and instead play out normally, thus incurring a decrease in prevalence of $z(u)=v$. This helps motivate the effect of the term $1-\prob(f_G^v(x,y)=1)-\prob(f_G^u(x,y)=1)$. The magnitude of this effect scales with the relative frequency with which $u$ is encountered as a game position, which corresponds to $\prob(u\in\text{{Path}}_G(x,y))$.

\begin{proof}[Proof of Lemma~\ref{lm:sampled-individual}]
    First, we will introduce some notation to assist with this proof. Let us write $r=\prob(u\in\text{Path}_G(x,y))$, $s_u=\prob(f_G^u(x,y)=1)$, and $s_v=\prob(f_G^v(x,y)=1)$. Let us also write
    \begin{align*}
        A&=\{w\in V\setminus\{u\}:\text{there is a directed path from $w$ to $u$}\},\\
        B&=\{w\in V\setminus\{u\}:\text{there is a directed path from $u$ to $w$}\},
    \end{align*}
    and note that $A$, $B$, and $\{u\}$ are pairwise disjoint sets. Finally, if we have $\text{Path}(x,y)=v_0v_1\ldots v_\ell$ when regarded as a directed path (where here and throughout we drop the subscript from $\text{Path}_G$ to simplify notation), then we will define
    \begin{align*}
        \text{Path}^1(x,y)&=\{v_i:\text{$i$ is even}\},\\
        \text{Path}^2(x,y)&=\{v_i:\text{$i$ is odd}\}.
    \end{align*}
    Note that because $\text{Path}(x,y)$ is the disjoint union of $\text{Path}^1(x,y)$ and $\text{Path}^2(x,y)$, we have
    \begin{equation}\label{eq:r-1-2}
        r=\prob(u\in\text{Path}^1(x,y))+\prob(u\in\text{Path}^2(x,y)).
    \end{equation}
    
    The event $z(u)=v$ can be written as the disjoint union of the following six events.
    \begin{align*}
        E_1&=u\notin\text{Path}(x,y)\wedge f_G(x,y)=1 \wedge x(u)=v\\
        E_2&=u\notin\text{Path}(x,y)\wedge f_G(x,y)=-1 \wedge y(u)=v\\
        E_3&=u\in\text{Path}^1(x,y)\wedge x(u)=v\wedge f_G^{v}(y,x)=-1\\
        E_4&=u\in\text{Path}^1(x,y)\wedge f_G^{u}(x,y)=-1\wedge y(u)=v\\
        E_5&=u\in\text{Path}^2(x,y)\wedge y(u)=v\wedge f_G^{v}(x,y)=-1\\
        E_6&=u\in\text{Path}^2(x,y)\wedge f_G^{u}(y,x)=-1\wedge x(u)=v
    \end{align*}
    Let us examine the probability of each of these events occurring. For $E_1$, the event $u\notin\text{Path}(x,y)\wedge f_G(x,y)=1$ can be determined using only $(x(w))_{w\neq u}$ and $(y(w))_{w\neq u}$, and so is independent of the event $x(u)=v$. Similarly, in $E_2$ the event $u\notin\text{Path}(x,y)\wedge f_G(x,y)=-1$ is independent of the event $y(u)=v$. Therefore,
    \begin{align}
        \prob(E_1)+\prob(E_2)=&\,\prob(u\notin\text{Path}(x,y)\wedge f_G(x,y)=1)\cdot p(u,v) \nonumber\\
        &+\prob(u\notin\text{Path}(x,y)\wedge f_G(x,y)=-1)\cdot p(u,v) \nonumber\\
        =&\,\prob(u\notin\text{Path}(x,y))\cdot p(u,v) \nonumber\\
        =&\,(1-r)\cdot p(u,v). \label{eq:E12}
    \end{align}
    For $E_3$, the event $u\in\text{Path}^1(x,y)$ can be determined using only $(x(w))_{w\in A}$ and $(y(w))_{w\in A}$, and the event $f_G^v(x,y)=1$ can be determined using only $(x(w))_{w\in B}$ and $(y(w))_{w\in B}$. Therefore, all three component events in $E_3$ are independent of each other. The same is also true of $E_5$. Therefore, noting that $\prob(f_G^v(x,y)=-1)=\prob(f_G^v(y,x)=-1)$, we can write
    \begin{align}
        \prob(E_3)+\prob(E_5)=&\,\prob(u\in\text{Path}^1(x,y))\cdot p(u,v)\cdot\prob(f_G^v(y,x)=-1) \nonumber\\
        &+\prob(u\in\text{Path}^2(x,y))\cdot p(u,v)\cdot\prob(f_G^v(x,y)=-1) \nonumber\\
        =&\,(\prob(u\in\text{Path}^1(x,y))+\prob(u\in\text{Path}^2(x,y)))\cdot p(u,v)\cdot\prob(f_G^v(x,y)=-1) \nonumber\\
        \overset{\eqref{eq:r-1-2}}{=}&\,r\cdot p(u,v)\cdot(1-s_v). \label{eq:E35}
    \end{align}
    For $E_4$, the event $u\in\text{Path}^1(x,y)$ can be determined using only $(x(w))_{w\in A}$ and $(y(w))_{w\in A}$, and the event $f_G^{u}(x,y)=-1$ can be determined using only $(x(w))_{w\in \{u\}\cup B}$ and $(y(w))_{w\in B}$. Therefore, all three component events in $E_4$ are independent of each other. The same is also true of $E_6$. Therefore, noting that $\prob(f_G^{u}(x,y)=-1)=\prob(f_G^{u}(y,x)=-1)$, we can write
    \begin{align}
        \prob(E_4)+\prob(E_6)=&\,\prob(u\in\text{Path}^1(x,y))\cdot\prob(f_G^{u}(x,y)=-1)\cdot p(u,v) \nonumber\\
        &+\prob(u\in\text{Path}^1(x,y))\cdot\prob(f_G^{u}(y,x)=-1)\cdot p(u,v) \nonumber\\
        =&\,(\prob(u\in\text{Path}^1(x,y))+\prob(u\in\text{Path}^2(x,y)))\cdot\prob(f_G^{u}(x,y)=-1)\cdot p(u,v) \nonumber\\
        \overset{\eqref{eq:r-1-2}}{=}&\,r\cdot(1-s_u)\cdot p(u,v). \label{eq:E46}
    \end{align}
    We can now combine these observations to obtain
    \begin{align*}
        \prob(z(u)=v)&=\sum_{i\in[6]}\prob(E_i)\overset{\eqref{eq:E12},\eqref{eq:E35},\eqref{eq:E46}}{=}(1-r)\cdot p(u,v)+r\cdot p(u,v)(1-s_v)+r\cdot(1-s_u)p(u,v)\\
        &=p(u,v)\cdot[1+r\cdot(1-s_v-s_u)],
    \end{align*}
    as required.
\end{proof}

As an aside, we note here a parallel with evolutionary game theory. Consider the discrete time replicator equation with nonlinear payoff functions (see (2.1) of~\cite{S-discrete-time-replicator-dynamics}; also \cite{HS-evolutionary-game-dynamics} for the more standard continuous and linear versions),
\begin{equation}\label{eq:replicator-dynamics}
    q_i'=q_i(1+a_i-\textstyle\sum_{j}q_ja_j),
\end{equation}
where we interpret $q_i$ as the proportion of type $i$ in a population and $a_i$ as the fitness of a type $i$ individual. The following proposition demonstrates that by identifying $q_i$ and $a_i$ appropriately, \eqref{eq:sampled-individual} can be seen to be of the form provided by \eqref{eq:replicator-dynamics}.
\begin{proposition}
    In the setting of Lemma~\ref{lm:sampled-individual}, let $u\in V$ be fixed and enumerate $F(u)=\{v_1,\ldots,v_k\}$. Let us identify
    \begin{align*}
        q_i&=p(u,v_i)\\
        a_i&=\prob(u\in\ntext{Path}_G(x,y))\cdot(1-\prob(f_G^{v_i}(x,y)=1)).
    \end{align*}
    Then \eqref{eq:sampled-individual} can be rewritten as $q_i'=q_i(1+a_i-\sum_{j\in[k]}q_ja_j)$.
\end{proposition}
\begin{proof}
First, note the identity
\begin{equation}\label{eq:rep-step}
    \prob(f_G^{u}(x,y)=1)=\sum_{j\in[k]}p(u,v_j)\cdot\prob(f_G^{v_j}(y,x)=-1)=\sum_{j\in[k]}q_j\cdot(1-\prob(f_G^{v_j}(x,y)=1)).
\end{equation}
Therefore,
\begin{align*}
    \prob(z(u)=v_i)&\overset{\eqref{eq:sampled-individual}}{=}p(u,v_i)\cdot[1+\prob(u\in\ntext{Path}_G(x,y))\cdot(1-\prob(f_G^{v_i}(x,y)=1)-\prob(f_G^u(x,y)=1))]\\
    &=q_i\cdot[1+a_i-\prob(u\in\ntext{Path}_G(x,y))\cdot\prob(f_G^u(x,y)=1)]\\
    &\overset{\eqref{eq:rep-step}}{=}q_i\cdot[1+a_i-\prob(u\in\ntext{Path}_G(x,y))\cdot\textstyle\sum_{j\in[k]}q_j\cdot(1-\prob(f_G^{v_j}(x,y)=1))]\\
    &=q_i(1+a_i-\textstyle\sum_{j\in[k]}q_ja_j),
\end{align*}
as required.
\end{proof}

In this sense, when executing Algorithm~\ref{alg:UMDA} on a game $G$, the evolution of the distribution $p(u,\blank)$ at each vertex $u$ of $G$ stochastically emulates these replicator dynamics. However, a key difference is that in standard evolutionary game theory, each fitness function $a_i:=a_i(q_1,\ldots,q_k)$ typically depends only on the distribution of types in the population; whereas the expression $a_i=\prob(u\in\text{Path}_G(u,v_i))\cdot(1-\prob(f_G^{v_i}(x,y)=1))$ depends on the distribution of `types' not just at the node $u$, but also at possibly all other nodes as well, and so the dynamics of each node cannot be considered in isolation.

\section{Switchability}\label{sect:switchability}

In Section~\ref{sect:introduction} we noted that our main result implies a probabilistic upper bound of $n^{O(\overline{s})}$ on an impartial combinatorial game, where $\overline{s}$ is an (often small) invariant of the corresponding game graph. In this section, we define this invariant and prove a key lemma.

Rather than defining this property, which we call \emph{switchability}, for a game as a whole, we will actually define switchability as a property $s(u)$ of each vertex $u$ in the game's vertex set $V$. Then later we will take $\overline{s}=\max_{u\in V}s(u)$ (see Corollary~\ref{cor:impartial-games}). Intuitively, $s(u)$ measures the `smallest' possible set of edges $A\subseteq E(G)$ such that any pair of strategies $x,y\in\XG$ satisfying $A\subseteq\{(v,x(v)):v\in V\}$ and $A\subseteq\{(v,y(v)):v\in V\}$ must also satisfy $u\in \text{Path}_G(x,y)$. The motivation is that if $x,y\sim\UnivDist{\XG}{p}$ for some $p\in\prod_{v\in\text{Int}(G)}\mathcal{P}_\gamma(F(v))$, then $u\in\text{Path}_G(x,y)$ is assured by having $x$ and $y$ take certain values at the vertices appearing at the tail of some edge in $A$, which occurs with probability at least $\gamma^{2s(u)}$. A property that places a lower bound on $\prob(u\in\text{Path}_G(x,y))$ in such a way will be very useful as we seek to apply Lemma~\ref{lm:sampled-individual} later.

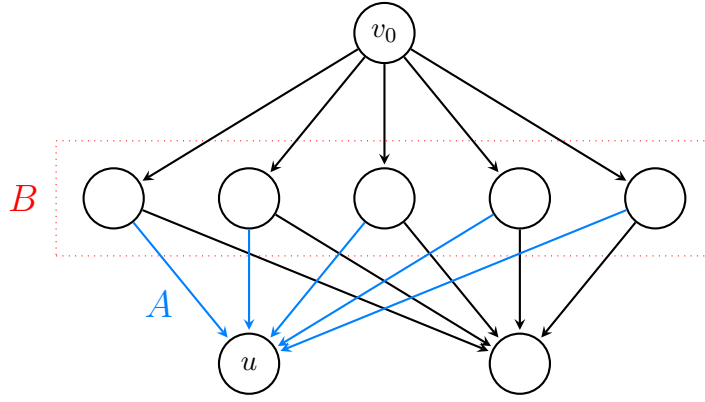
\begin{figure}
    \begin{center}
    \scalebox{1.0}{
    \begin{tikzpicture}[
            node distance = 1.8cm, 
            global edge style
        ]

        \node[state] (b1) {} ;
        \foreach \i in {2,...,5}{
            \pgfmathtruncatemacro{\prev}{\i-1}
            \node[state,right of=b\prev] (b\i) {} ;
        }

        \node[draw,red,dotted,inner sep=10pt,fit=(b1) (b5)] {};
        \node[red] at ($(b1)+(-1.2,0)$) () {\Large $B$} ;

        \tikzset{node distance=2.2cm}
        \node[state,above of=b3] (v0) {$v_0$} ;
        \node[state,below of=b2] (u) {$u$} ;
        \node[state,below of=b4] (w) {} ;

        \foreach\i in {1,...,5}{
            \draw[->,black] (v0) edge (b\i) ;
            \draw[->,black] (b\i) edge (w) ;
            \draw[->,capri] (b\i) edge (u) ;
        }
        \node[capri] at ($(b1)+(0.6,-1.4)$) () {\Large $A$} ;
        
    \end{tikzpicture}
    }
    \end{center}
    \caption{An example of switchability.}
    \label{fig:switchability-motivation}
\end{figure}

In the description above, a naive approach would be to take `smallest' to simply mean having fewest edges. However, while this gives a working definition, when then bounding $\prob(u\in\text{Path}_G(x,y))$ below, it is clear that significant improvements can be made in many cases. Consider the example shown in Figure~\ref{fig:switchability-motivation}. Our naive approach suggests that if $x,y\sim\UnivDist{\XG}{p}$ for some $p\in\prod_{v\in\text{Int}(G)}\mathcal{P}_\gamma(F(v))$, then $\prob(u\in\text{Path}_G(x,y))\geqslant\gamma^{10}$. However, it would be better to observe that $\prob(u\in\text{Path}_G(x,y))\geqslant\gamma$, as visiting $u$ can be assured a single choice to move to $u$ made by the player who makes the first move after reaching the layer $B$ (in this case, always the player $y$).

To better capture this notion, we will not take `smallest' to mean fewest edges, but rather smallest depth, defined in the following way (we recall here that all graphs are assumed to be acyclic).

\begin{definition}
    Given a set of edges $A\subseteq E(G)$, we define the \emph{depth} of $A$ to be
    \begin{equation*}
        \text{\emph{Depth}}(A)=\max{\{|A\cap E(P)|:\text{$P$ is a directed path in $G$}\}}.
    \end{equation*}
\end{definition}

With this, the full description of switchability is provided by the following two definitions.

\begin{definition}\label{def:v-switcher}
    Given a set $A\subseteq E(G)$, we (inductively) say that a directed path $P=v_0\ldots v_\ell$ is \emph{$A$-compatible} if any of the following conditions hold.
    \stepcounter{capitalcounter}
    \begin{enumerate}[label = {\bfseries \emph{\Alph{capitalcounter}\arabic{enumi}}}]
        \item\label{A-comp-0} $P=v_0$.
        \item\label{A-comp-1} $v_0\ldots v_{\ell-1}$ is $A$-compatible and $v_{\ell-1}v_\ell\in A$.
        \item\label{A-comp-2} $v_0\ldots v_{\ell-1}$ is $A$-compatible and there is no $w\in V$ such that $v_{\ell-1}w\in A$.
    \end{enumerate}
    Then, given a vertex $v$, we say that $A$ is a \emph{$v$-switcher} if $v$ is contained in every $A$-compatible directed path $v_0\ldots v_\ell$ with $v_\ell\notin\text{\emph{Int}}(G)$.
\end{definition}

\begin{definition}\label{def:switchability}
    The \emph{switchability} $s(v)$ of a vertex $v$ is the smallest possible depth of a $v$-switcher. We will also write $\overline{s}=\max_{v\in V}s(v)$.
\end{definition}

\begin{figure}
    \begin{center}
    \scalebox{1.0}{
    \begin{tikzpicture}[
            node distance = 1.8cm, 
            global edge style
        ]

        \node[state] (v0) {$v_0$} ;
        \foreach \i in {1,...,7}{
            \pgfmathtruncatemacro{\prev}{\i-1}
            \ifthenelse{\i=5}
                {\tikzstyle{inedge}=[->,capri]}
                {\tikzstyle{inedge}=[->,black]}
            \node[state,right of=v\prev] (v\i) {\ifthenelse{\i=5}{$v$}{}} ;
            \draw[inedge] (v\prev) edge (v\i) ;
            \ifthenelse{\i=1}{}{
                \pgfmathtruncatemacro{\pprev}{\i-2}
                \draw[inedge] (v\pprev) edge[out=60,in=120] (v\i) ;
            }
        }
        \node[capri] at ($(v4)+(0,1.7)$) () {\Large $A$} ; 
        
    \end{tikzpicture}
    }

    \vspace{5mm}

    \scalebox{1.0}{
    \begin{tikzpicture}[
            node distance = 2cm, 
            global edge style
        ]

        \node[state] (v01) {$v_0$} ;
        \foreach \i in {1,...,4}
        {
            \pgfmathtruncatemacro{\prev}{\i-1}
            \node[state,right of=v\prev1] (v\i1) {} ;
            \node[state,below of=v\i1] (v\i0) {} ;
            \node[state,above of=v\i1] (v\i2) {\ifthenelse{\i=3}{$v$}{}} ;
        }
        \draw[->,capri] (v01) edge (v12) ;
        \draw[->] (v01) edge (v11) ;
        \draw[->] (v01) edge (v10) ;
        \foreach \i in {2,...,4}
        {
            \pgfmathtruncatemacro{\prev}{\i-1}
            \ifthenelse{\i=3}
                {\tikzstyle{topinedge}=[->,capri]}
                {\tikzstyle{topinedge}=[->,black]}
            \draw[->] (v\prev0) edge (v\i0) ;
            \draw[->] (v\prev1) edge (v\i0) ;
            \draw[->] (v\prev0) edge (v\i1) ;
            \draw[->] (v\prev2) edge (v\i1) ;
            \draw[topinedge] (v\prev1) edge (v\i2) ;
            \draw[topinedge] (v\prev2) edge (v\i2) ;
        }
        \node[capri] at ($(v22)+(0.8,0.7)$) () {\Large $A$} ;

        \path (v01) ++(0,-0.25) coordinate (p0);
        \path (v12) ++(0,-0.25) coordinate (p1);
        \path (v21) ++(0,-0.25) coordinate (p2);
        \path (v32) ++(0,-0.25) coordinate (p3);
        \path (v42) ++(0,-0.25) coordinate (p4);

        \draw[red,dashed,thick] (p0) -- (p1) -- (p2) -- (p3) -- (p4);
        
    \end{tikzpicture}
    }
    \end{center}
    \caption{Two illustrations of switchability. In the first, $s(v)=1$, and a $v$-switcher of depth $1$ is shown in blue. In the second, $s(v)=2$, a $v$-switcher of depth $2$ is shown in blue, and one example of an $A$-compatible path is shown in red.}
    \label{fig:switchability-examples}
\end{figure}
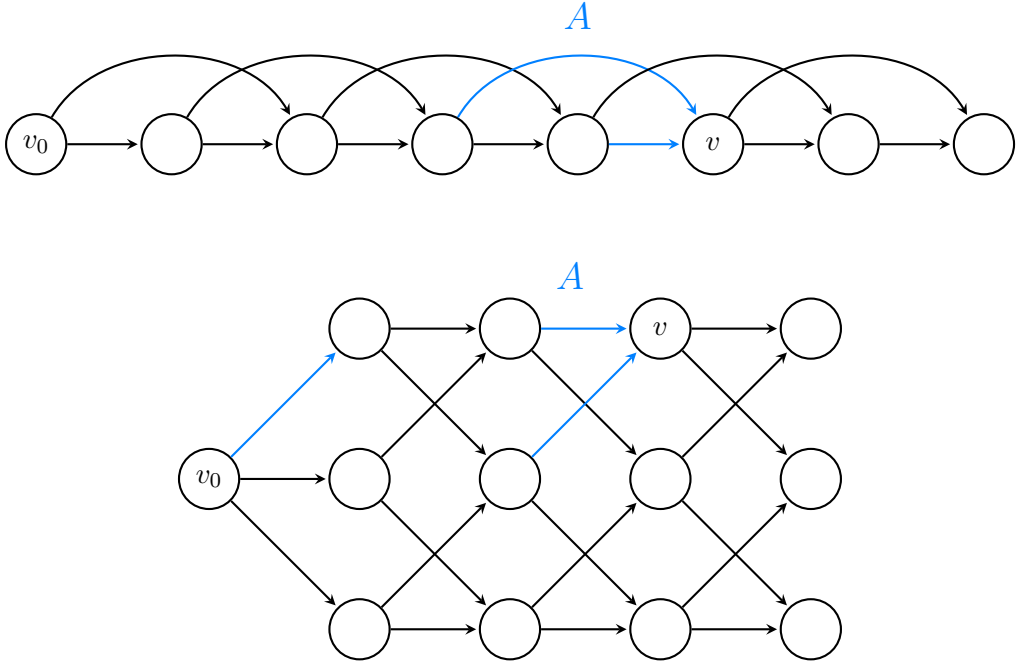

Thus, while the set $A$ shown in Figure~\ref{fig:switchability-motivation} has 5 edges, it has $\text{Depth}(A)=1$, and so in that case we have $s(u)=1$. Figure~\ref{fig:switchability-examples} shows two further illustrations of switchability. For certain games, constructing a small $v$-switcher is quite straightforward (see proof of Proposition~\ref{prop:nim} later); in other cases where determining switchability is not obvious, the following upper bound may be used instead.

\begin{proposition}\label{prop:switchability-bound}
    If there is a directed path of length $\ell$ from $v_0$ to $v$, then $s(v)\leqslant\ell$.
\end{proposition}
\begin{proof}
    If $P$ is a directed path from $v_0$ to $v$, then every $E(P)$-compatible path $v_0\ldots v_\ell$ with $v_\ell\notin\text{Int}(G)$ includes $P$ as a prefix, and hence also includes $v$. Thus, $E(P)$ is a $v$-switcher of depth $\ell$.
\end{proof}

To complete this section, the required lower bound on $\prob(u\in\text{Path}_G(x,y))$ is provided by the following lemma. Note that as well as improving the naive approach by using $\text{Depth}(A)$ instead of $|A|$, we also deduce a result of $\gamma^{s(v)}$ instead of $\gamma^{2s(v)}$ by carefully accounting for the fact that at most one player can visit each possible game position (due to the previous assumption that the impartial combinatorial games considered are acyclic).

\begin{lemma}\label{lm:switchability}
    Suppose that $p\in\prod_{v\in\text{\emph{Int}}(G)}\mathcal{P}_\gamma(F(v))$ and $x,y\sim\eUnivDist{\XG}{p}$. Then for every $v\in V$, $\prob(v\in\text{\emph{Path}}_G(x,y))\geqslant\gamma^{s(v)}$.
\end{lemma}
\begin{proof}
    The distribution of $\text{Path}_G(x,y)$ is the same as the random set $P$ produced by the following process.
    \begin{enumerate}
        \item Initially, set $z_0=v_0$.
        \item For $i\geqslant 0$, do the following.
        \begin{enumerate}
            \item If $z_i\notin\text{Int}(G)$, then set $P=\{z_0,\ldots,z_i\}$.
            \item Otherwise if $z_i\in\text{Int}(G)$, sample $z_{i+1}\sim p(z_i)$.
        \end{enumerate}
    \end{enumerate}
    We will generate an instance of the above process in a very specific way using a collection of independent $\text{Unif}([0,1])$ random variables. First, let $A$ be a $v$-switcher of depth $s(v)$, and let $B=\{u\in V:\text{$(u,w)\in A$ for some $w\in F(u)$}\}$. Next, for each $u\in V$, let $\phi_u:[0,1]\to F(u)$ be any function satisfying the following properties.
    \stepcounter{capitalcounter}
    \begin{enumerate}[label = {\bfseries \Alph{capitalcounter}\arabic{enumi}}]
        \item\label{switchmap1} If $X\sim\text{Unif}([0,1])$, then $\prob(\phi_u(X)=w)=p(u,w)$ for every $u\in V$.
        \item\label{switchmap2} If $s\in[0,\gamma]$ and $u\in B$, then $(u,\phi_u(s))\in A$.
    \end{enumerate}
    Note that this is possible because $p(u,w)\geqslant\gamma$ holds for every $u\in V$ and $w\in F(u)$. The modified process is then as follows.
    \begin{enumerate}
        \setcounter{enumi}{-1}
        \item Let $X_1,\ldots,X_n,Y_1,\ldots,Y_n$ be independent $\text{Unif}([0,1])$ random variables.
        \item Initially, set $z_0=v_0$.
        \item For $i\geqslant 0$, do the following
        \begin{enumerate}
            \item If $z_i\notin\text{Int}(G)$, then set $Q=\{z_0,\ldots,z_i\}$.
            \item If $z_i\in B$, then set $r_i=|\{z_0,\ldots,z_i\}\cap B|$ and $z_{i+1}=\phi_{z_i}(X_{r_i})$.
            \item If $z_i\in\text{Int}(G)\setminus B$, then set $r_i=|\{z_0,\ldots,z_i\}\setminus B|$ and $z_{i+1}=\phi_{z_i}(Y_{r_i})$.
        \end{enumerate}
    \end{enumerate}
    From \ref{switchmap1} it follows that $Q$ has the same distribution as $P$, and hence also as $\text{Path}_G(x,y)$.

    We now claim that $v\in Q$ whenever $X_1,\ldots,X_{s(v)}\in[0,\gamma]$. The key observation is that under this regime, the first $s(v)$ visits that $Q$ makes to $B$ must be followed by an edge in $A$. Let us label $Q=z_0z_1\ldots z_\ell$. We will show by induction on $i$ that $z_0\ldots z_i$ is $A$-compatible for every $i\in[\ell]$, noting that the case $i=0$ holds because $z_0=v_0$. For the inductive step, if $z_0\ldots z_i$ is $A$-compatible, then the only way for $z_0\ldots z_{i+1}$ to not be $A$-compatible is to have $z_i\in B$ and $(z_i,z_{i+1})\notin A$. But then, because the first $s(v)$ visits that $Q$ makes to $B$ are followed by an edge in $A$, we can infer that $z_0\ldots z_i$ already includes at least $s(v)$ edges in $A$. Letting $w\in V$ be such that $(z_i,w)\in A$, we then have that $R:=z_0\ldots z_iw$ is a directed path with $|E(R)\cap A|\geqslant s(v)+1$, a contradiction to the depth of $A$. So in fact the inductive step holds, and $z_0\ldots z_i$ is $A$-compatible for every $i$. In particular, $Q$ is $A$-compatible, and hence $v\in Q$.

    Thus, $v\in Q$ whenever $X_1,\ldots,X_{s(v)}\in[0,\gamma]$, and hence
    \begin{align*}
        \prob(v\in\text{Path}_G(x,y))=\prob(v\in Q)\geqslant\prob(X_1,\ldots,X_{s(v)}\leqslant\gamma)=\gamma^{s(v)},
    \end{align*}
    as required.
\end{proof}

\section{Main result}\label{sect:main-result}

In order to state runtime results, we adopt the standard black box convention where runtime is defined as the number of times a function is queried until the algorithm reaches the desired search objective (see~\cite{DJW-blackbox}), as follows.

\begin{definition}
    Suppose that $G$ is an impartial combinatorial game, and that $\mathcal{A}$ is an algorithm which makes $\tau$ queries of $f_G$ during each generation. Then, given a set $B\subseteq\XG$, the \emph{runtime} of $\mathcal{A}$ on $f_G$ is defined to be the random variable
    \begin{equation*}
        T_\mathcal{A}^G(B)=\tau\cdot\min{\{t:P_t\cap B\neq\emptyset\}}
    \end{equation*}
    where $P_t\subseteq\XG$ is the population of $\mathcal{A}$ at the start of generation $t$. (If the game $G$ is clear from context, we will write $T_\mathcal{A}$ instead of $T_\mathcal{A}^G$.)
\end{definition}

Our main result is now provided by Theorem~\ref{thm:upper-bound-impartial-combinatorial}. In simple terms, it states that if Algorithm~\ref{alg:UMDA} is executed on an impartial combinatorial game $G$ using a sufficiently large population size $\mu$, then with high probability its runtime is at most $O(\mu\cdot r(G))$, where $r(G)$ is a formula of the game graph expressed in terms of its number of vertices $n$, maximum degree $\Delta$, and a summation involving the switchability $s(v)$ (Definition~\ref{def:switchability}) at each critical position $v\in W_G$ (Definition~\ref{def:critical-positions}). Notably, $r(G)$ is increasing in each of $n$, $\Delta$, and $s(v)$, indicating that games for which these quantities are high may be the most difficult to optimise. We remark that the exact parameter settings for Algorithm~\ref{alg:UMDA} appearing in the statement have not been chosen to guarantee an optimal runtime, but rather to make the proof more comprehensible.

\begin{theorem}\label{thm:upper-bound-impartial-combinatorial}
    There is a constant $C>0$ such that the following holds. Let $G$ be an $n$-vertex impartial combinatorial game with maximum degree $\Delta$, and let $\hat{s}=\max_{v\in W_G}s(v)$.
    Let $K>0$, and let $\mathcal{A}$ be described by Algorithm~\ref{alg:UMDA}, where $\gamma=1/(20\Delta n)$ and
    \begin{equation}\label{eq:mu-lower}
        \mu\geqslant C(K+\hat{s}+1)(20\Delta n)^{1+2\hat{s}}\log{n}.
    \end{equation}
    Then,
    \begin{equation*}
        \prob\Biggl[T_\mathcal{A}^G(\text{\emph{Opt}}(G))\geqslant C\mu \sum_{v\in W_G}(20\Delta n)^{s(v)}\log{n}\Biggr]\leqslant n^{-K}.
    \end{equation*}
\end{theorem}

The asymptotic behaviour of the runtime bound may not be immediately obvious from the form stated here. Accordingly, we will shortly provide an easier to digest corollary using the facts $s(v)\leqslant\hat{s}\leqslant\max_{v\in V}s(v)$ and $|W_G|\leqslant n$ to remove the role of $W_G$ and the corresponding summation. For many games (including the applications considered later), this simplified bound has the same asymptotic behaviour as the one provided by Theorem~\ref{thm:upper-bound-impartial-combinatorial}. Nonetheless, as it is possible to construct games for which Theorem~\ref{thm:upper-bound-impartial-combinatorial} offers significant improvement of the simplified bound, we opt to retain the more general form above.

We will now briefly provide some intuition for the proof of Theorem~\ref{thm:upper-bound-impartial-combinatorial}. As characterised by Lemma~\ref{lm:optimality-characterisation}, we know that any strategy $x\in\XG$ that makes the correct decision at every critical position $v\in W_G$ is an element of $\text{Opt}(G)$ (where here, making a correct decision means ensuring $x(v)$ has a Sprague-Grundy value of $0$). Thus, we consider the sequence $p_0,p_1,\ldots$ appearing in Algorithm~\ref{alg:UMDA}, and estimate the time until the algorithm arrives at some $p$ such that, with high probability, an $x\sim\UnivDist{\XG}{p_t}$ makes the correct decision at every critical position. The progress to arrive at such a $p$ is effectively broken down into $|W_G|$ steps: fixing an ordering $u_1,\ldots,u_n$ of $V$ such that $F(u_i)\subseteq\{u_1,\ldots,u_{i-1}\}$ for every $i\in[n]$ (a reverse topological ordering), step $k$ finishes when, with high probability, an $x\sim\UnivDist{\XG}{p_t}$ makes the correct decision at the first $k$ critical positions appearing in the ordering. Bounding the length of time to complete step $k$ is accomplished by combining Lemmas~\ref{lm:optimality-characterisation},~\ref{lm:sampled-individual}, and~\ref{lm:switchability} to show that if sampled individuals are usually making the correct decision at the first $k$ critical positions in the ordering, then the algorithm has a bias towards retaining individuals who also make the correct decision at the next critical position in the ordering. Note that this step-by-step process does not appear explicitly in the proof, but is implicit from the definition of a function $\hat{g}$ measuring progress towards the optimality condition (see \eqref{eq:hat-g-def}).

\begin{proof}[Proof of Theorem~\ref{thm:upper-bound-impartial-combinatorial}]
    First, let us introduce some further notation. Given a set $V'\subseteq V$ we will write $p_t(u,V')=\sum_{v\in V'}p(u,v)$. Recalling that Algorithm~\ref{alg:UMDA} ensures that $p_t\in\prod_{v\in \text{Int}(G)}\mathcal{P}_\gamma(F(v))$ at every step, we will write $\mathcal{Q}=\prod_{v\in\text{Int}(G)}\mathcal{P}_\gamma(F(v))$. Let $h:V\to\mathbb{N}_0$ denote the Sprague-Grundy function, and let $V_0=h^{-1}(\{0\})$ and $V_1=V\setminus V_0$. We also will assume that $n\geqslant3$, as any impartial combinatorial game $G$ with $n<3$ satisfies $\text{Opt}(G)=\XG$ (such games satisfy $\Delta\leqslant1$, and so in fact $|\XG|=1$ in these cases).
    
    Let $u_1,\ldots,u_n$ be an ordering of $V$ such that $F(u_i)\subseteq\{u_1,\ldots,u_{i-1}\}$ for every $i\in[n]$ (note that such an ordering exists as $G$ is assumed to be acyclic). Let us write $A_i$ for the set of $p\in\mathcal{Q}$ such that $p(u,V_1)\leqslant\sfrac{1}{10n}$ for all $u\in W_G\cap\{u_1,\ldots,u_i\}$. If for some generation $t$ we have $p_t\in A_n$, then for every $v\in W_G$ we have
    \begin{equation}\label{eq:q-opt-bound}
        q_t(v)(V_0)\overset{\eref{pi-3}}{\geqslant}\pi_\gamma^{F(v)}(q_t(v))(V_0)-\gamma|F(v)|=p_t(v,V_0)-\gamma|F(v)|\geqslant1-\sfrac{1}{10n}-\gamma\Delta>1-\sfrac{1}{5n}.
    \end{equation}
    Recalling from Lemma~\ref{lm:optimality-characterisation} that if $x\in\XG\setminus\text{Opt}(G)$ then $x(v)\in V_1$ for some $v\in W_G$, we can deduce
    \begin{align*}
        |\{j\in[\mu]:P_t(j)\notin\text{Opt}(G)\}|&\leqslant\sum_{v\in W_G}|\{j\in[\mu]:P_t(j)(v)\in V_1\}|=\sum_{v\in W_G}\mu\cdot q_t(v)(V_1)\\
        &=\sum_{v\in W_G}\mu\cdot(1-q_t(v)(V_0))\overset{\eqref{eq:q-opt-bound}}{<}\frac{|W_G|\mu}{5n}\leqslant\frac{\mu}{5}<\mu,
    \end{align*}
    and hence $P_t\cap\text{Opt}(G)\neq\emptyset$. In particular, if $T^\ast=\min{\{t:p_t\in A_n\}}$ then $T_\mathcal{A}^G(\text{Opt}(G))\leqslant \mu\cdot T^\ast$.

    We will define a map $\hat{g}:\mathcal{Q}\to\mathbb{R}_{\geqslant0}$ that will measure progress towards $A_n$. To do this, first let $g:[\gamma,1-\gamma]\to\mathbb{R}_{\geqslant 0}$ be given by
    \begin{equation*}
        g(y)=\log{\left(\frac{y}{1-y}\right)}-\log{\left(\frac{\gamma}{1-\gamma}\right)},
    \end{equation*}
    so that $g$ is a monotone increasing function. Then, given $p\in\mathcal{Q}$, let $\ell(p)=\max{\{i\in[n]:p\in A_i\}}$ and define
    \begin{equation}\label{eq:hat-g-def}
        \hat{g}(p)=
        \begin{cases}
            \scalebox{0.85}{$\displaystyle\sum_{i\in[\ell(p)]}\boldone(u_i\in W_G)\cdot \left(g(1-\gamma)\cdot\left(\frac{32}{\gamma^{s(u_i)}}\right)+1\right)+g(p(u_{\ell(p)+1},V_0))\cdot\left(\frac{32}{\gamma^{s(u_{\ell(p)+1})}}\right)$}&\text{if $\ell(p)<n$,}\\
            \scalebox{0.85}{$\displaystyle\sum_{i\in[\ell(p)]}\boldone(u_i\in W_G)\cdot \left(g(1-\gamma)\cdot\left(\frac{32}{\gamma^{s(u_i)}}\right)+1\right)$}&\text{if $\ell(p)=n$.}
        \end{cases}
    \end{equation}
    Define also $g_\text{max}=\max_{p\in\mathcal{Q}}\hat{g}(p)$, and note that $p\in A_n$ if and only if $\hat{g}(p)=g_\text{max}$. The motivation for the function $\hat{g}:\mathcal{Q}\to\mathbb{R}_{\geqslant0}$ is that the value of $\hat{g}(p_t)$ increases as $p_t$ moves through $\mathcal{Q}$ towards $A_n$. Indeed, the first term of \eqref{eq:hat-g-def} is a summation depending on $\ell(p)$ only; its role ensures that $g(p)\geqslant g(p')$ whenever $p\in A_i$ and $p'\notin A_i$, and hence $\hat{g}$ increases true to the sequence $A_0\supseteq A_1\supseteq\ldots\supseteq A_n$. The second term measures progress within some $A_i$ as we move towards $A_{i+1}$ (it increases as the value of $p_t(u_i,V_1)$ decreases toward $\sfrac{1}{10n}$).
    
    Denote $X_t(i)=g(p_t(u_i,V_0))$. We will later show the following two claims, where the second is a direct consequence of the first.
    \begin{claim}\label{clm:cg-logit-drift}
        If $p_t\in A_{i-1}$ and $u_i\in W_G$, then
        \begin{equation*}
            \prob(X_{t+1}(i)\leqslant \min{\{g(1-2\Delta\gamma),X_{t}(i)+\gamma^{s(u_i)}/32\}})\leqslant n^{-K-3\hat{s}-5}.
        \end{equation*}
    \end{claim}
    \begin{claim}\label{clm:cg-logit-drift-corollary}
        If $p_t\notin A_n$, then $\prob(\hat{g}(p_{t+1})\geqslant \hat{g}(p_t)+1)\geqslant 1-n^{-K-3\hat{s}-4}$.
    \end{claim}

    Claim~\ref{clm:cg-logit-drift-corollary} asserts that if $p_t$ has not yet reached $A_n$, then we should expect the value of $\hat{g}(p_t)$ to increase by at least $1$ in the next generation. However, $\hat{g}(p_t)$ cannot increase by at least $1$ more than $\lfloor g_\text{max}\rfloor$ times. Precisely, if $T^\ast> g_\text{max}$ then we must have $p_t\notin A_n$ and $\hat{g}(p_t)< g_\text{max}$ for every $t\leqslant g_\text{max}$. In particular, it would then hold that $\hat{g}(p_{t})<\hat{g}(p_{t-1})+1$ for some $t\in[\lfloor g_\text{max}\rfloor]$. Therefore, using a union bound with Claim~\ref{clm:cg-logit-drift-corollary}, we have
    \begin{equation}\label{eq:T-ast-tau}
        \prob[T^\ast>g_\text{max}]\leqslant g_\text{max}\cdot n^{-K-3\hat{s}-4}.
    \end{equation}
    Noting (using Lemma~\ref{lm:g-simple}) that
    \begin{equation}\label{eq:g-gamma}
        g(1-\gamma)\overset{\eref{g-simp-3}}{\leqslant}2\log{(1/\gamma)}=2\log{(20\Delta n)}\leqslant 5\log{n}-1,
    \end{equation}
    we can bound
    \begin{align}
        g_\text{max}=\sum_{v\in W_G}\left(g(1-\gamma)\cdot\left(\frac{32}{\gamma^{s(v)}}\right)+1\right)&\overset{\eqref{eq:g-gamma}}{\leqslant}\sum_{v\in W_G}(5\log{n})\cdot32\cdot(20\Delta n)^{s(v)} \nonumber\\
        &\hspace{0.1cm}< C\sum_{v\in W_G}(20\Delta n)^{s(v)}\log{n} \label{eq:df1}\\
        &\hspace{0.1cm}\leqslant C|W_G|(20\Delta n)^{\hat{s}}\log{n}\leqslant n^{3\hat{s}+4}. \label{eq:df2}
    \end{align}
    We now have
    \begin{align*}
        \prob\Bigl[T_\mathcal{A}^G(\text{Opt}(G))\geqslant C\mu &\sum_{v\in W_G}(20\Delta n)^{s(v)}\log{n}\Bigr]\leqslant\prob\Bigl[T^\ast\geqslant C \sum_{v\in W_G}(20\Delta n)^{s(v)}\log{n}\Bigr]\\
        &\overset{\eqref{eq:df1}}{\leqslant}\prob\left[T^\ast> g_\text{max}\right]\overset{\eqref{eq:T-ast-tau}}{\leqslant}g_\text{max}\cdot n^{-K-3\hat{s}-4}\overset{\eqref{eq:df2}}{\leqslant} n^{3\hat{s}+4}\cdot n^{-K-3\hat{s}-4}= n^{-K},
    \end{align*}
    as required. Therefore, all that remains is to prove Claims~\ref{clm:cg-logit-drift} and~\ref{clm:cg-logit-drift-corollary}.

    \begin{proof}[Proof of Claim~\ref{clm:cg-logit-drift}]
    Assume $x,y\sim\UnivDist{\XG}{p_t}$ are independent. To assist with this claim, we will introduce some further notation. Let $r=\prob(u_i\in\text{Path}_G(x,y))$, and note that from Lemma~\ref{lm:switchability} we have
    \begin{equation}\label{eq:r-bound}
        r\geqslant\gamma^{s(u_i)}.
    \end{equation}
    Given $w\in V$, let us write $N_w$ as a shorthand for the event $f_G^w(x,y)=1$, and note that because $x$ and $y$ are independent and identically distributed,
    \begin{equation}\label{eq:Nw-equiv}
        \prob(N_w)=\prob(f_G^w(x,y)=1)=\prob(f_G^w(y,x)=1).
    \end{equation}
    Finally, let us also write $F_0=F(u_i)\cap V_0$ and $F_1=F(u_i)\cap V_1$.

    Given $v\in V$, we wish to consider $\prob(z(u_i)=v)$, where $z$ is the winner of the game $G$ played between $x$ and $y$ (as in Lemma~\ref{lm:sampled-individual}). This will be useful, as the individuals $P_{t+1}(1),\ldots,P_{t+1}(\mu)$ selected in lines~\ref{ln:selection-loop-start}-\ref{ln:selection-loop-end} of Algorithm~\ref{alg:UMDA} are independent and with the same distribution as $z$, and hence for every $v\in V$,
    \begin{equation}\label{eq:q-dist}
        \mu\cdot q_{t+1}(u_i,v)\sim\Bin{\mu}{\prob(z(u_i)=v)}.
    \end{equation}
    To analyse $\prob(z(u_i)=v)$, first note that we have
    \begin{equation}\label{eq:Nui-expansion}
        1-\prob(N_{u_i})=\prob(f_G^{u_i}(x,y)=-1)=\sum_{w\in F(u_i)}p_t(u_i,w)\prob(f_G^w(y,x)=1)\overset{\eqref{eq:Nw-equiv}}{=}\sum_{w\in F(u_i)}p_t(u_i,w)\prob(N_w).
    \end{equation}
    Therefore, applying Lemma~\ref{lm:sampled-individual} with $u=u_i$,
    \begin{align}
        \prob(z(u_i)=v)&\overset{\eqref{eq:sampled-individual}}{=}p_t(u_i,v)\cdot[1+r\cdot(1-\prob(N_v)-\prob(N_{u_i}))]\nonumber\\
        &\overset{\eqref{eq:Nui-expansion}}{=}p_t(u_i,v)\cdot\Bigl[1+r\cdot\Bigl(-\prob(N_v)+\sum_{w\in F(u_i)}p_t(u_i,w)\prob(N_w)\Bigr)\Bigr] \nonumber\\
        &=p_t(u_i,v)\cdot\Bigl[1+r\cdot\Bigl(-\prob(N_v)+\sum_{w\in F_0}p_t(u_i,w)\prob(N_w)+\sum_{w\in F_1}p_t(u_i,w)\prob(N_w)\Bigr)\Bigr].\label{eq:cg-sel-press}
    \end{align}
    In particular, we also have
    \begin{equation}\label{eq:c2-F}
        \prob(z(u_i)=v)\overset{\eqref{eq:cg-sel-press}}{\geqslant} p(u_i,v)\cdot[1-r].
    \end{equation}

    Next, we would like to place some simple bounds on $\prob(N_w)$ for $w\in F_0\cup F_1$. If $w\in\{u_1,\ldots,u_{i-1}\}$ satisfies $h(w)\neq0$, then by using the fact that $p_t\in A_{i-1}$,
    \begin{align*}
        \prob(N_w)&=\prob(f_G^w(x,y)=1)\overset{\text{Lemma~\ref{lm:optimality-characterisation}}}{\geqslant}\prob(\text{$h(x(v))=0$ for all $v\in W_G\cap\{u_1,\ldots,u_{i-1}\}$})\\
        &=\prod_{v\in W_G\cap\{u_1,\ldots,u_{i-1}\}}p_t(v,V_0)=\prod_{v\in W_G\cap\{u_1,\ldots,u_{i-1}\}}(1-p_t(v,V_1))\\
        &\geqslant(1-\sfrac{1}{10n})^n\geqslant\sfrac{9}{10}.
    \end{align*}
    On the other hand, if $w\in\{u_1,\ldots,u_{i-1}\}$ satisfies $h(w)=0$, then by using the fact that $p_t\in A_{i-1}$,
    \begin{align*}
        1-\prob(N_w)&=\prob(f_G^w(x,y)=-1)\overset{\text{Lemma~\ref{lm:optimality-characterisation}}}{\geqslant}\prob(\text{$h(y(v))=0$ for all $v\in W_G\cap\{u_1,\ldots,u_{i-1}\}$})\geqslant\sfrac{9}{10}.
    \end{align*}
    In summary, we have
    \begin{align}
        \prob(N_w)\leqslant\sfrac{1}{10}&\qquad\text{whenever $w\in F_0$},\label{eq:Nw-F0}\\
        \prob(N_w)\geqslant\sfrac{9}{10}&\qquad\text{whenever $w\in F_1$}.\label{eq:Nw-F1}
    \end{align}

    Finally, we will apply Corollary~\ref{cor:binom-bernstein} to establish that certain events occur with very low probability. A straightforward numerical manipulation we will use after each application is that for every $v\in F(u_i)$, because $p_t(u_i,v)\geqslant \gamma$,
    \begin{align}\label{eq:bb-manip}
        \exp{\left(-\frac{r^2\mu p_t(u_i,v)/16}{8(1+r/4)}\right)}&\overset{\eqref{eq:r-bound}}{\leqslant}\exp{\left(-\frac{\gamma^{2s(u_i)+1}\mu}{200}\right)}\leqslant\exp{\left(-\frac{\gamma^{2\hat{s}+1}\mu}{200}\right)}\nonumber\\
        &\overset{\eqref{eq:mu-lower}}{\leqslant}\exp{\left(-\frac{C(K+\hat{s}+1)\log{n}}{200}\right)}\leqslant\frac{1}{2}n^{-K-3\hat{s}-6}.
    \end{align}
    
    We now complete the proof of the claim by dividing into two cases. Note that the properties \eref{bb1}-\eref{bb3} and \eref{g-simp-1}-\eref{g-simp-2} quoted hereafter are from the results of Section~\ref{app:preliminary-results}.
    
    \textbf{Case 1: $p_t(u_i,F_1)\leqslant\sfrac{1}{2}$.} If $v\in F_1$, then
    \begin{align}
        \prob(z(u_i)=v)&\overset{\eqref{eq:cg-sel-press},\eqref{eq:Nw-F0},\eqref{eq:Nw-F1}}{\leqslant} p_t(u_i,v)\cdot[1+r\cdot(-\sfrac{9}{10}+\sfrac{1}{10}p_t(u_i,F_0)+p_t(u_i,F_1))]\nonumber\\
        &\hspace{0.7cm}\leqslant p_t(u_i,v)\cdot[1-\sfrac{1}{4}r].\label{eq:c1-F1}
    \end{align}
    By using \eqref{eq:q-dist} and \eqref{eq:c1-F1} to apply \eref{bb2} with $\alpha=r/4$, it holds for any fixed $v\in F_1$ that
    \begin{equation*}
        \prob(q_{t+1}(u_i,v)>(1-r/8)p_t(u_i,v))\leqslant\exp{\left(-\frac{r^2\mu p_t(u_i,v)/16}{8(1+r/4)}\right)}\overset{\eqref{eq:bb-manip}}{\leqslant}\frac{1}{2}n^{-K-3\hat{s}-6}\leqslant n^{-K-3\hat{s}-6}.
    \end{equation*}
    Therefore, by taking a union bound over $F_1$, it occurs with probability at least $1-n^{-K-3\hat{s}-5}$ that
    \begin{align}
        &q_{t+1}(u_i,v)\leqslant(1-r/8)p_t(u_i,v)&\text{for every $v\in F_1$,}\label{eq:c1-hp1}
    \end{align}
    and so we proceed under the assumption that this occurs. Note that this automatically gives us for any $v\in F_1$ that
    \begin{equation}\label{eq:pt-on-F1}
        p_{t+1}(u_i,v)=\pi_\gamma^{F(u_i)}(q_{t+1}(u_i,\blank))(v)\overset{\eref{pi-2}}{\leqslant}\max{\{\gamma,q_{t+1}(u_i,v)\}}\overset{\eqref{eq:c1-hp1}}{\leqslant}\max{\{\gamma,p_t(u_i,v)\}}=p_t(u_i,v).
    \end{equation}
    So if $p_t(u_i,F_1)\leqslant2\Delta\gamma$ then $p_{t+1}(u_i,F_1)\leqslant2\Delta\gamma$ and hence $X_{t+1}(i)=g(p_{t+1}(u_i,V_0))=g(p_{t+1}(u_i,F_0))=g(1-p_{t+1}(u_i,F_1))\geqslant g(1-2\Delta\gamma)$. On the other hand, if $p_t(u_i,F_1)\geqslant2\Delta\gamma$ then there is some $v\in F_1$ such that $p(u_i,v)\geqslant \sfrac{1}{\Delta}p_t(u_i,F_1)\geqslant 2\gamma$, and hence
    \begin{align}
        p_{t+1}(u_i,F_1)&=p_{t+1}(u_i,v)+p_{t+1}(u_i,F_1\setminus\{v\})\overset{\eref{pi-2}}{\leqslant}\max{\{\gamma,q_{t+1}(u_i,v)\}}+p_{t+1}(u_i,F_1\setminus\{v\}) \nonumber\\
        &\hspace{-0.45cm}\overset{\eqref{eq:c1-hp1},\eqref{eq:pt-on-F1}}{\leqslant}\max{\{\gamma,(1-r/8)p_t(u_i,v)\}}+p_t(u_i,F_1\setminus\{v\}) \nonumber\\
        &=(1-r/8)p_t(u_i,v)+p_t(u_i,F_1\setminus\{v\}) \nonumber\\
        &=p_t(u_i,F_1)-(r/8)p_t(u_i,v)\overset{\eqref{eq:r-bound}}{\leqslant}(1-\gamma^{s(u_i)}/8)p_t(u_i,F_1). \label{eq:p-adjust-1}
    \end{align}
    In particular, this would then imply that
    \begin{align*}
        X_{t+1}(i)&=g(p_{t+1}(u_i,V_0))=g(1-p_{t+1}(u_i,V_1))\overset{\eqref{eq:p-adjust-1}}{\geqslant}g(1-(1-(\gamma^{s(u_i)}/8))p_t(u_i,V_1))\\
        &\overset{\eref{g-simp-2}}{\geqslant}g(1-p_t(u_i,V_1))+\frac{\gamma^{s(u_i)}}{16}=X_t(i)+\frac{\gamma^{s(u_i)}}{16}.
    \end{align*}
    Combining the cases $p_t(u_i,F_1)\leqslant 2\Delta\gamma$ and $p_t(u_i,F_1)\geqslant2\Delta\gamma$ shows that the event $X_{t+1}(i)\geqslant\min{\{g(1-2\Delta\gamma),X_t(i)+\gamma^{s(u_i)}/32\}}$ holds with probability at least $1-n^{-K-3\hat{s}-5}$.

    \textbf{Case 2: $p_t(u_i,F_1)\geqslant\sfrac{1}{2}$.} If $v\in F_0$, then
    \begin{align}
        \prob(z(u_i)=v)&\overset{\eqref{eq:cg-sel-press},\eqref{eq:Nw-F0},\eqref{eq:Nw-F1}}{\geqslant} p_t(u_i,v)\cdot[1+r\cdot(-\sfrac{1}{10}+\sfrac{9}{10}p_t(u_i,F_1))]\nonumber\\
        &\hspace{0.7cm}\geqslant p_t(u_i,v)\cdot[1+\sfrac{1}{4}r].\label{eq:c2-F0}
    \end{align}
    By using \eqref{eq:q-dist} and \eqref{eq:c2-F0} to apply \eref{bb1} with $\alpha=r/4$, it holds for every $v\in F_0$ that
    \begin{equation*}
        \prob(q_{t+1}(u_i,v)<(1+r/8)p_t(u_i,v))\leqslant\exp{\left(-\frac{r^2\mu p_t(u_i,v)/16}{8(1+r/4)}\right)}\overset{\eqref{eq:bb-manip}}{\leqslant}\frac{1}{2}n^{-K-3\hat{s}-6}.
    \end{equation*}
    By using \eqref{eq:q-dist} and \eqref{eq:c2-F} to apply \eref{bb3} with $\alpha=r$, it holds for every $v\in F(u_i)$ that
    \begin{equation*}
        \prob(q_{t+1}(u_i,v)<(1-2r)p_t(u_i,v))\leqslant\exp{\left(-\frac{r^2\mu p_t(u_i,v)/16}{8(1+r/4)}\right)}\overset{\eqref{eq:bb-manip}}{\leqslant}\frac{1}{2}n^{-K-3\hat{s}-6}.
    \end{equation*}
    Therefore, by taking a union bound over $F_0$ and also $F(u_i)$, it occurs with probability at least $1-n^{-K-3\hat{s}-5}$ that
    \begin{align}
        &q_{t+1}(u_i,v)\geqslant(1+r/8)p_t(u_i,v)&\text{for every $v\in F_0$,}\label{eq:c2-hp1}\\
        &q_{t+1}(u_i,v)\geqslant(1-2r)p_t(u_i,v)&\text{for every $v\in F(u_i)$,}\label{eq:c2-hp2}
    \end{align}
    and so we proceed under the assumption that this occurs. Recalling that $\gamma=1/(20\Delta n)$ and the assumption that $n\geqslant3$, we can now bound $\beta_\gamma^-(q_{t+1}(u_i,\blank))$ above as
    \begin{align}
        \beta_\gamma^-(q_{t+1}(u_i,\blank))&=\sum_{v\in F(u_i)}\max{\{\gamma-q_{t+1}(u_i,v),0\}}\overset{\eqref{eq:c2-hp2}}{\leqslant}\sum_{v\in F(u_i)}\max{\{\gamma-(1-2r)\gamma,0\}}\nonumber\\
        &\leqslant\sum_{v\in F(u_i)}2r\gamma\leqslant2\Delta r\gamma=\frac{r}{10n}\leqslant\frac{r}{30}\leqslant\frac{r}{24}\cdot(1-\sfrac{1}{60})\leqslant\displaystyle\frac{r}{24}\cdot(1-\gamma\Delta).\label{eq:b-minus}
    \end{align}
    Hence, using that $q_{t+1}(u_i,v)\geqslant\gamma$ for every $v\in F_0$,
    \begin{align}
        p_{t+1}(u_i,F_0)&=\pi_\gamma^{F(u_i)}(q_{t+1}(u_i,\blank))(F_0)\overset{\eref{pi-1}}{\geqslant}\left(1-\frac{\beta_\gamma^-(q_{t+1}(u_i,\blank))}{1-\gamma\Delta}\right)q_{t+1}(u_i,F_0) \nonumber\\
        &\hspace{-0.4cm}\overset{\eqref{eq:b-minus},\eqref{eq:c2-hp1}}{\geqslant}(1-r/24)(1+r/8)p_t(u_i,F_0)\overset{\eqref{eq:r-bound}}{\geqslant}(1+\gamma^{s(u_i)}/16)p_t(u_i,F_0). \label{eq:p-adjust-2}
    \end{align}
    This would then imply that
    \begin{align*}
        X_{t+1}(i)&=g(p_{t+1}(u_i,F_0))\overset{\eqref{eq:p-adjust-2}}{\geqslant}g((1+\gamma^{s(u_i)}/16)p_t(u_i,F_0))\\
        &\overset{\eref{g-simp-1}}{\geqslant}g(p_t(u_i,F_0))+\frac{\gamma^{s(u_i)}}{32}=X_t(i)+\frac{\gamma^{s(u_i)}}{32}.
    \end{align*}
    Thus, the event $X_{t+1}(i)\geqslant\min{\{g(1-2\Delta\gamma),X_t(i)+\gamma^{s(u_i)}/32\}}$ holds with probability at least $1-n^{-K-3\hat{s}-5}$.
    \renewcommand{\qedsymbol}{$\boxdot$}
    \end{proof}
    \renewcommand{\qedsymbol}{$\square$}

    \begin{proof}[Proof of Claim~\ref{clm:cg-logit-drift-corollary}]
    Suppose that $p_t\notin A_n$, so that $\ell:=\ell(p_t)<n$. For every $i\in[\ell]$ with $u_i\in W_G$, it follows from the fact that $p_t\in A_i$ that $p(u_i,V_0)=1-p(u_i,V_1)\geqslant1-\sfrac{1}{10n}$ and hence
    \begin{equation}\label{eq:X-level}
        X_t(i)=g(p(u_i,V_0))\geqslant g(1-\sfrac{1}{10n})=g(1-2\Delta\gamma).
    \end{equation}
    Let $I=\{i\in[\ell+1]:u_i\in W_G\}$ so that $\ell+1\in I$. For $i\in I$, let $E_i$ be the event that
    \begin{equation*}
        X_{t+1}(i)\geqslant\min{\{g(1-2\Delta\gamma),X_t(i)+\gamma^{s(u_i)}/32\}}.
    \end{equation*}
    We will now show that if $E_i$ holds for every $i\in I$, then $\hat{g}(p_{t+1})\geqslant \hat{g}(p_t)+1$. Indeed, if $E_i$ holds for every $i\in I$, then it follows from \eqref{eq:X-level} that $X_{t+1}(i)\geqslant g(1-2\Delta\gamma)$ for every $i\in[\ell]$ with $u_i\in W_G$, and hence $p_{t+1}\in A_{\ell}$. If additionally $\ell(p_{t+1})>\ell$, then $\hat{g}(p_{t+1})\geqslant \hat{g}(p_t)+1$ is immediate from~\eqref{eq:hat-g-def}. On the other hand, if $\ell(p_{t+1})=\ell$, then $E_{\ell+1}$ implies $X_{t+1}(\ell+1)\geqslant X_t(\ell+1)+\gamma^{s(u_{\ell+1})}/32$ and hence,
    \begin{equation*}
        \hat{g}(p_{t+1})-\hat{g}(p_t)=(X_{t+1}(\ell+1)-X_t(\ell+1))\cdot\left(\frac{32}{\gamma^{s(u_{\ell+1})}}\right)\geqslant1.
    \end{equation*}
    Therefore, using a union bound we have
    \begin{align*}
        \prob(\hat{g}(p_{t+1})\geqslant\hat{g}(p_t)+1)\geqslant\prob(\wedge_{i\in I}E_i)\geqslant1-\sum_{i\in I}\prob(E_i^c)\overset{\text{Claim~\ref{clm:cg-logit-drift}}}{\geqslant}1-|I|\cdot n^{-K-3\hat{s}-5}\geqslant1-n^{-K-3\hat{s}-4},
    \end{align*}
    as required.
    \renewcommand{\qedsymbol}{$\boxdot$}
    \qedhere
    \renewcommand{\qedsymbol}{$\square$}
    \qedsymbol
    \end{proof}
    \renewcommand{\qedsymbol}{}
\end{proof}
\renewcommand{\qedsymbol}{$\square$}

For many applications, rather than applying Theorem~\ref{thm:upper-bound-impartial-combinatorial} directly it will be convenient to use the following corollary.

\begin{corollary}\label{cor:impartial-games}
    There is a constant $C>0$ such that the following holds. Let $G$ be an impartial combinatorial game with maximum degree $\Delta$, and let $\overline{s}=\max_{v\in V}s(v)$. Let $K>0$, and assume $\mathcal{A}$ uses parameters $\gamma=1/(20\Delta n)$ and $\mu= C(K+\overline{s}+1)(20\Delta n)^{1+2\overline{s}}\log{n}$. Then,
    \begin{equation*}
        \prob[T_\mathcal{A}^G(\text{\emph{Opt}}(G))\geqslant C^2(K+\overline{s}+1)(20\Delta n)^{2+3\overline{s}}\log^2{n}]\leqslant n^{-K}.
    \end{equation*}
\end{corollary}
\begin{proof}
    By noting that
    \begin{equation*}
        \hat{s}=\max_{v\in W_G}s(v)\leqslant\max_{v\in V}s(v)=\overline{s},
    \end{equation*}
    and
    \begin{equation*}
        \sum_{v\in W_G}(20\Delta n)^{s(v)}\leqslant\sum_{v\in V}(20\Delta n)^{s(v)}\leqslant\sum_{v\in V}(20\Delta n)^{\overline{s}}=n\cdot(20\Delta n)^{\overline{s}}\leqslant(20\Delta n)^{\overline{s}+1},
    \end{equation*}
    this is an immediate consequence of Theorem~\ref{thm:upper-bound-impartial-combinatorial}.
\end{proof}

\section{Applications}\label{sect:applications}

\newcommand{\TA}[1]{T_\mathcal{A}(\text{\emph{Opt}}(#1))}

In this section we will apply Theorem~\ref{thm:upper-bound-impartial-combinatorial} to obtain several runtimes for Algorithm~\ref{alg:UMDA} on a number of well-established combinatorial games. Throughout, we state runtimes in terms of $n$, the number of possible game positions, and always assume that $\mathcal{A}$ is described by Algorithm~\ref{alg:UMDA}. All described games are played under the normal play convention (that a player unable to move loses), as established in Section~\ref{sect:impartial-combinatorial-games}.

\subsection{Subtraction Nim}\label{sect:subtraction-nim}

Nim is a strategic game in which players take turns removing items from distinct heaps. Variants have been played across cultures since ancient history~\cite{R-prehistory-of-nim,Y-matchsticks}, and it was also the game of choice for some of the earliest machines and computers dedicated to game playing~\cite{J-nimbi,M-nimatron,R-machine-for-nim}. Nim is also perhaps the most important impartial combinatorial game from a mathematical perspective, with the Sprague-Grundy theorem establishing that, for a particular formulation of equivalence which characterises strategic continuation, every position in any impartial combinatorial game is equivalent to some position of a one-heap game of Nim~\cite{C-numbers-and-games}.

While the version central to combinatorial game theory typically allows players to remove any positive number of items on their turn, here we consider the well-known one-heap variant in which there is an upper limit on the number of items that can be taken at once (see, for example, \cite{F-games-of-no-chance-chapter}). Given parameters $n$ and $k$, $\textsc{SubtractionNim}_n^k$ begins with an initial heap of $(n-1)$ items, and on each turn a player may remove between $1$ and $k$ items from the heap. The game graph for $\textsc{SubtractionNim}_n^2$ is shown in Figure~\ref{fig:switchability-examples}. This game constitutes the simplest example of a \emph{subtraction game}~\cite{G-impartial} of also a \emph{take-away game}~\cite{S-takeaway}, both of which are expansive and well-studied classes of impartial combinatorial games. We have the following polynomial runtime for $\textsc{SubtractionNim}_n^k$.
\begin{proposition}\label{prop:nim}
    $\emph{\textsc{SubtractionNim}}_n^k$ satisfies $\overline{s}\leqslant1$ and $\Delta\leqslant k$. Thus, for each $K>0$ there exists $C>0$ such that for appropriately chosen parameters in Algorithm~\ref{alg:UMDA},
    \begin{equation*}
        \prob[\TA{\emph{\textsc{SubtractionNim}}_n^k}\geqslant C(kn)^5\log^2{n}]\leqslant n^{-K}.
    \end{equation*}
\end{proposition}
\begin{proof}
    For $\textsc{SubtractionNim}_n^k$ we have $V=\{0,1,\ldots,n-1\}$, $v_0=n-1$, and $F(v)=\{v-1,\ldots,v-k\}\cap V$. Note that $V\setminus\text{Int}(G)=\{0\}$.

    We need to verify that $s(v)\leqslant1$ for every $v\in V$. Given $v$, let $A_v=\{(v+i,v):i\in[k-1]\}\cap E(G)$. We have $\text{Depth}(A_v)=1$, as any directed path in $G$ can visit $v$ at most once. To see that $A_v$ is a $v$-switcher, suppose that $z_0\ldots z_\ell$ is an $A_v$-compatible directed path from $z_0=n-1$ to $V\setminus\text{Int}(G)=\{0\}$. Because at most $k$ items are removed on each turn, there is some $i$ such that $z_i\in\{v,v+1,\ldots,v+(k-1)\}$. But then we either have $z_i=v$ or, in order for $z_0\ldots z_{i+1}$ to remain $A_v$-compatible, $z_iz_{i+1}\in A_v$ and hence $z_{i+1}=v$. In either case, we deduce that $v$ lies on every $A_v$-compatible directed path from $n-1$ to $0$. Thus, $A_v$ is a $v$-switcher of depth $1$, and hence $s(v)\leqslant 1$.

    From this, we have that $\overline{s}\leqslant1$. Combined with the observation that $\Delta\leqslant k$, the result then follows from Corollary~\ref{cor:impartial-games}.
\end{proof}

\subsection{Silver Dollar}

We consider the variant of Silver Dollar played without the eponymous silver dollar \cite{C-numbers-and-games,FB-silver-dollar,G-impartial}; however, it should be noted that Theorem~\ref{thm:upper-bound-impartial-combinatorial} also implies a similar polynomial runtime for the original version of Silver Dollar attributed to de Bruijn (see also~\cite{C-numbers-and-games}).

Given parameters $m$ and $k$, $\textsc{SilverDollar}_m^k$ is played using $k$ coins on a horizontal strip of $m$ squares, with the coins initially placed on the rightmost $k$ squares (most descriptions actually have the coins placed on arbitrary starting squares, however this does not significantly affect our analysis).  A turn consists of moving one coin leftwards any number of spaces, provided the coin does not go past any other coins. In addition, coins may never occupy the same square. Assuming $k$ is a fixed constant, the number of game positions is $n=\binom{m}{k}$. We have the following polynomial runtime for $\textsc{SilverDollar}_m^k$.

\begin{proposition}\label{prop:silver-dollar}
    Let $k\in\mathbb{N}$ be fixed. $\emph{\textsc{SilverDollar}}_m^k$ satisfies $\overline{s}\leqslant k$ and $\Delta\leqslant m=O(n^{1/k})$. Thus, for each $K>0$ there exists $C>0$ such that for appropriately chosen parameters in Algorithm~\ref{alg:UMDA}, 
    \begin{equation*}
        \prob[\TA{\emph{\textsc{SilverDollar}}_m^k}\geqslant Cn^{5+3k+(2/k)}\log^2{n}]\leqslant n^{-K}.
    \end{equation*}
\end{proposition}
\begin{proof}
    On each turn, for each empty square there is at most one possible move that places a coin onto that square. Therefore, $\Delta\leqslant m-k\leqslant m=O(n^{1/k})$. Next, any possible game position can be reached from the starting position in at most $k$ moves (simply move each coin in order from left to right onto the required square). Therefore, using Proposition~\ref{prop:switchability-bound}, we have $\overline{s}\leqslant k$. The required result then follows from Corollary~\ref{cor:impartial-games} using these bounds on $\Delta$ and $\overline{s}$.
\end{proof}

\subsection{Turning Turtles}

Here we consider one instance of a large class of coin turning games~\cite{BCG-winning-ways-3,G-impartial}. Given a parameter $m$, $\textsc{TurningTurtles}_m$ is played using a row of $m$ coins, initially all showing heads. A turn consists of turning over one coin from heads to tails, and then optionally turning over one more coin anywhere to the left of that one (regardless of whether it is showing heads or tails). Play continues until all coins show tails. Noting that the total number of game positions is $n=2^m$, we have the following quasipolynomial runtime for $\textsc{TurningTurtles}_m$.

\begin{proposition}\label{prop:turning-turtles}
    $\emph{\textsc{TurningTurtles}}_m$ satisfies $\overline{s}\leqslant\log_2{n}$ and $\Delta\leqslant(\log_2{n})^2$. Thus, for each $K>0$ there exists $c>0$ such that for appropriately chosen parameters in Algorithm~\ref{alg:UMDA},
    \begin{equation*}
        \prob[\TA{\emph{\textsc{TurningTurtles}}_m}\geqslant n^{c\log{n}}]\leqslant n^{-K}.
    \end{equation*}
\end{proposition}
\begin{proof}
    On each turn, there are at most $m$ possible moves that turn over only one coin and at most $\binom{m}{2}$ possible moves that turn over two coins. Therefore, $\Delta\leqslant m+\binom{m}{2}\leqslant m^2$. Next, any possible game position can be reached from the starting position in at most $m$ moves (simply turn over the required coins from heads to tails one by one). Therefore, using Proposition~\ref{prop:switchability-bound} we have $\overline{s}\leqslant m$. Noting that $m=\log_2{n}$, and hence
    \begin{equation*}
        C^2(K+\overline{s}+1)(20\Delta n)^{2+3\overline{s}}\log^2{n}\leqslant C^2(K+\log_2{n}+1)(20(\log_2{n})^2n)^{2+3\log_2{n}}(\log{n})^2\leqslant n^{c\log{n}},
    \end{equation*}
     the required result then follows from Corollary~\ref{cor:impartial-games}.
\end{proof}

\subsection{Chomp}\label{sect:chomp}

Since its introduction by Schuh~\cite{S-chomp} and later by Gale~\cite{G-chomp}, Chomp has inspired a great deal of theoretical and empirical analysis, as well as numerous variants incorporating, for example, graphs and simplicial complexes~\cite{GKM-chomp-on-graphs}. While typically played on any rectangular board, we focus on square instances for the sake of conciseness.

Given a parameter $m$, $\textsc{Chomp}_m$ is played on an $m\times m$ board. A turn consists of removing one square, as well as all squares to the right and above. However, if a player removes the square in the lower-left corner (the `poison' square) they immediately lose. Note that to instantiate this game under our normal play convention, we can make removing the lower-left corner fatal by simply removing the position that has no remaining squares. We can establish the following quasipolynomial runtime for $\textsc{Chomp}_m$.

\begin{proposition}\label{prop:chomp}
    $\emph{\textsc{Chomp}}_m$ satisfies $\overline{s}\leqslant O(\log_2{n})$ and $\Delta\leqslant O((\log_2{n})^2)$. Thus, for each $K>0$ there exists $c>0$ such that for appropriately chosen parameters in Algorithm~\ref{alg:UMDA},
    \begin{equation*}
        \prob[\TA{\emph{\textsc{Chomp}}_m}\geqslant n^{c\log{n}}]\leqslant n^{-K}.
    \end{equation*}
\end{proposition}
\begin{proof}
    In each possible game position, every row must be at least as long as the row above it. In particular, there is a correspondence between game positions and lattice paths (i.e., paths that only move right and down along the squares' edges) from the top-left corner to bottom-right corner, with the path marking out the boundary of the remaining squares. Using stars and bars counting (see \cite[Theorem~8.5.1]{B-introductory-combinatorics} for a full treatment) and removing the position that has no remaining squares, the total number of game positions is $n=\binom{2m}{m}-1=\Theta(4^m/\sqrt{m})$. On each turn, there are at most $m^2$ moves available, and so $\Delta\leqslant m^2$. Next, any possible game position can be reached from the starting position in at most $m$ moves (simply make the appropriate chomp row by row working from top to bottom). Therefore, using Proposition~\ref{prop:switchability-bound}, we have $\overline{s}\leqslant m$. Thus, as with Proposition~\ref{prop:turning-turtles}, we have $\Delta\leqslant m^2$ and $\overline{s}\leqslant m$ where $m=O(\log{n})$, and so the result follows.
\end{proof}

\section{Concluding remarks}\label{sect:concluding-remarks}

\begin{figure}
    \begin{center}
    \scalebox{1.0}{
    \begin{tikzpicture}[
            node distance = 1.8cm, 
            global edge style
        ]

        \node[state] (v0) {$v_0$} ;
        \foreach \i in {1,...,7}{
            \pgfmathtruncatemacro{\prev}{\i-1}
            \node[state,right of=v\prev] (v\i) {} ;
            \draw[->,black] (v\prev) edge (v\i) ;
        }
        \node[state] at ($(v3)+(0,-2.5)$) (u) {} ;
        \node[state,right of=u] (w) {} ;
        \foreach \i in {0,...,7}{
            \draw[->,black] (v\i) edge (u) ;
        }
        \draw[->,black] (u) edge (w) ;
        
    \end{tikzpicture}
    }
    \end{center}
    \caption{A game that should be easy to optimise, but contains vertices with switchability $\Theta(n)$.}
    \label{fig:front-to-back}
\end{figure}
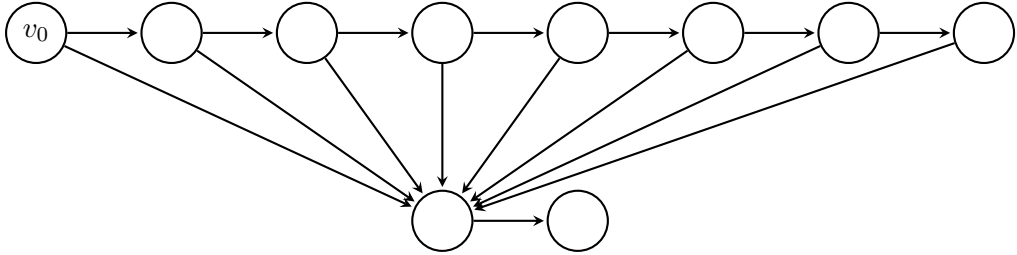

We conclude with some brief remarks about the main result and future work.

In order to accommodate the high degree of generality in Theorem~\ref{thm:upper-bound-impartial-combinatorial}, the proof makes a number assumptions about the route taken to the search objective. A notable one is that, if $v$ is the next critical position to be optimised, or one that has already been learned, then the probability $\prob(v\in\text{Path}_G(x,y))$ that $v$ is encountered in a game played out by sampled individuals $x,y$ is bound below by $\gamma^{s(v)}$. Lemma~\ref{lm:switchability} demonstrates that analysis of $\prob(v\in\text{Path}_G(x,y))$ is a major contribution to the eventual runtime, serving a role akin to a dynamic learning rate for the algorithm at position $v$. A key insight is that encountering a large range of game positions by evaluating diverse sets of opponents is essential to an algorithm's success. However, it is apparent that the general bound $\prob(v\in\text{Path}_G(x,y))\geqslant\gamma^{s(v)}$ could be greatly improved through closer analysis of coevolutionary dynamics, especially for specific games. For example, if individuals often misplay at a winning position $v$, opponents should begin to exploit this by steering the game towards $v$; the resulting feedback mechanism between $\prob(v\in\text{Path}_G(x,y))$ and $p_t(v,\blank)$ can assist more efficient learning.

A related assumption is that game positions are optimised sequentially, moving from the end of the game and working backwards, not unlike a recursive computation of the Sprague-Grundy function. However, this is not the route to optimality we would expect CoEAs to adopt for all games (consider Figure~\ref{fig:front-to-back}, where UMDA would naturally optimise starting from $v_0$ and working forwards). Moreover, because Lemma~\ref{lm:optimality-characterisation} is not a necessary condition, there is potential for CoEAs to demonstrate bias towards learning simpler elements of $\text{Opt}(G)$ without the need to implicitly deduce all zeros of the Sprague-Grundy function (for example, when played on a square board, there is an optimal strategy for Chomp that can be described by specifying an action at only $\Theta(m^2)$ of the $\Theta(4^m/\sqrt{m})$ game positions).

In future work, we aim to provide more detailed analysis related to both of the above assumptions in order to provide stronger runtime results on classes of impartial combinatorial games. A longer term goal is the development of runtime analysis applicable to game representations that are practical even for games with exponentially many positions, such as in situations encountered in genetic programming.

\bibliographystyle{abbrv}
\bibliography{references}

\appendix

\section{Preliminary results}\label{app:preliminary-results}

Here we provide two straightforward results that will be useful to quote throughout the proof of Theorem~\ref{thm:upper-bound-impartial-combinatorial}. The first is derived from the Chernoff bounds for binomial random variables given by Theorem~\ref{thm:chernoff}, which is in turn an immediate consequences of~\cite[Theorem~3.2]{CL-concentration}. We remark that the conclusions \eref{bb1}-\eref{bb3} have been optimised for ease of integration with the proofs in this paper, rather than tightness of bound.

\begin{theorem}\label{thm:chernoff}
    If $X\sim\text{\emph{Bin}}(m,q)$, then for any $t\geqslant0$ it holds that
    \begin{align}
        \prob(X\leqslant mq-t)&\leqslant\exp{\left(-\frac{t^2/2}{mq}\right)},\label{eq:chernoff-1} \\
        \prob(X\geqslant mq+t)&\leqslant\exp{\left(-\frac{t^2/2}{mq+t/3}\right)}.\label{eq:chernoff-2}
    \end{align}
\end{theorem}

\begin{corollary}\label{cor:binom-bernstein}
    Suppose $q\in[0,1/2]$ and $X\sim\text{\emph{Bin}}(\mu,q)$.
    \stepcounter{capitalcounter}
    \begin{enumerate}[label = {\emph{\bfseries \Alph{capitalcounter}\arabic{enumi}}}]
        \item\label{bb1} For any $\alpha>0$ and $p\in[0,1/2]$ satisfying $q\geqslant(1+\alpha)p$,
        \begin{equation*}
            \prob(X/\mu\leqslant(1+\alpha/2)p)\leqslant\exp{\left(-\frac{\alpha^2\mu p}{8(1+\alpha)}\right)}.
        \end{equation*}
        \item\label{bb2} For any $\alpha>0$ and $p\in[0,1/2]$ satisfying $q\leqslant(1-\alpha)p$,
        \begin{equation*}
            \prob(X/\mu\geqslant(1-\alpha/2)p)\leqslant\exp{\left(-\frac{\alpha^2\mu p}{8(1+\alpha)}\right)}.
        \end{equation*}
        \item\label{bb3} For any $\alpha>0$ and $p\in[0,1/2]$ satisfying $q\geqslant(1-\alpha)p$,
        \begin{equation*}
            \prob(X/\mu\leqslant(1-2\alpha)p)\leqslant\exp{\left(-\frac{\alpha^2\mu p/16}{8(1+\alpha/4)}\right)}.
        \end{equation*}
    \end{enumerate}
\end{corollary}
\begin{proof}
    For \eref{bb1}, let $Y_1\sim\text{Bin}(\mu,(1+\alpha)p)$ so that $X\gstochdom Y_1$. We then have
    \begin{align*}
        \prob(X/\mu\leqslant(1+\alpha/2)p)&=\prob(X\leqslant(1+\alpha/2)p\mu)\leqslant\prob(Y_1\leqslant(1+\alpha/2)p\mu)\\
        &=\prob(Y_1\leqslant(1+\alpha)p\mu-\alpha p \mu/2)\overset{\eqref{eq:chernoff-1}}{\leqslant}\exp{\left(-\frac{\alpha^2p^2\mu^2/8}{\mu(1+\alpha)p}\right)}\leqslant\exp{\left(-\frac{\alpha^2 p\mu}{8(1+\alpha)}\right)},
    \end{align*}
    as required. For \eref{bb2}, let $Y_2\sim\text{Bin}(\mu,(1-\alpha)p)$ so that $X\lstochdom Y_2$. We then have
    \begin{align*}
        \prob(X/\mu\geqslant(1-\alpha/2)p)&=\prob(X\geqslant(1-\alpha/2)p\mu)\leqslant\prob(Y_2\geqslant(1-\alpha/2)p\mu)\\
        &=\prob(Y_1\geqslant(1-\alpha)p\mu+\alpha p \mu/2)\overset{\eqref{eq:chernoff-2}}{\leqslant}\exp{\left(-\frac{\alpha^2p^2\mu^2/8}{\mu(1-\alpha)p+(\alpha p\mu/6)}\right)}\\
        &\leqslant\exp{\left(-\frac{\alpha^2 p\mu}{8(1+\alpha)}\right)},
    \end{align*}
    as required. For \eref{bb3}, let $Y_3\sim\text{Bin}(\mu,(1-\alpha)p)$ so that $X\gstochdom Y_3$. We then have
    \begin{align*}
        \prob(X/\mu\leqslant(1-2\alpha)p)&=\prob(X\leqslant(1-2\alpha)p\mu)\leqslant\prob(Y_3\leqslant(1-2\alpha)p\mu)\\
        &=\prob(Y_3\leqslant(1-\alpha)p\mu-\alpha p \mu)\overset{\eqref{eq:chernoff-1}}{\leqslant}\exp{\left(-\frac{\alpha^2p^2\mu^2/2}{\mu(1-\alpha)p}\right)}\leqslant\exp{\left(-\frac{\alpha^2 p\mu/16}{8(1+\alpha/4)}\right)},
    \end{align*}
    as required.
\end{proof}

\begin{lemma}\label{lm:g-simple}
    Given $\gamma\in[0,\sfrac{1}{2})$, let $g:[\gamma,1-\gamma]\to\mathbb{R}_{\geqslant 0}$ be given by
    \begin{equation*}
        g(y)=\log{\left(\frac{y}{1-y}\right)}-\log{\left(\frac{\gamma}{1-\gamma}\right)}.
    \end{equation*}
    Then, the following properties hold.
    \stepcounter{capitalcounter}
    \begin{enumerate}[label = {\emph{\bfseries \Alph{capitalcounter}\arabic{enumi}}}]
        \item\label{g-simp-1} If $y\in[\gamma,\sfrac{1}{2}]$ and $a\in[0,1)$, then $g((1+a)y)-g(y)\geqslant a/2$.
        \item\label{g-simp-2} If $y\in[\sfrac{1}{2},1-\gamma]$ and $a\in[0,1)$, then $g(1-(1+a)y)-g(1-y)\geqslant a/2$.
        \item\label{g-simp-3} $\max_{y\in[\gamma,1-\gamma]}g(y)\leqslant2\log{(1/\gamma)}$.
    \end{enumerate}
\end{lemma}
\begin{proof}
    \textbf{\eref{g-simp-1}:} If $y\in[\gamma,\sfrac{1}{2}]$ and $a\in[0,1)$, then
    \begin{align*}
        g((1+a)y)-g(y)&=\log{\left(\frac{(1+a)y}{1-(1+a)y}\right)}-\log{\left(\frac{y}{1-y}\right)}=\log{\left(\frac{(1+a)(1-y)}{1-(1+a)y}\right)}\\
        &=\log{\left(1+\frac{a}{1-(1+a)y}\right)}\geqslant\log{(1+a)}\geqslant a/2.
    \end{align*}

    \textbf{\eref{g-simp-2}:} If $y\in[\sfrac{1}{2},1-\gamma]$ and $a\in[0,1)$, then
    \begin{align*}
        g(1-(1-a)y)-g(1-y)&=\log{\left(\frac{1-(1-a)y}{(1-a)y}\right)}-\log{\left(\frac{1-y}{y}\right)}=\log{\left(\frac{1-(1-a)y}{(1-a)(1-y)}\right)}\\
        &=\log{\left(1+\frac{a}{(1-a)(1-y)}\right)}\geqslant\log{(1+a)}\geqslant a/2.
    \end{align*}

    \textbf{\eref{g-simp-3}:} Because $g$ is an increasing function,
    \begin{equation*}
        \max_{y\in[\gamma,1-\gamma]}g(y)=g(1-\gamma)=\log{\left(\frac{1-\gamma}{\gamma}\right)}-\log{\left(\frac{\gamma}{1-\gamma}\right)}=2\log{\left(\frac{1-\gamma}{\gamma}\right)}\leqslant2\log{(1/\gamma)},
    \end{equation*}
    as required.
\end{proof}

\end{document}